\documentclass[hidelinks,11pt]{article} 
\usepackage{fullpage}
\usepackage{microtype}
\usepackage{graphicx}
\usepackage{subfig}
\usepackage{booktabs} 
\usepackage{hyperref}
\usepackage{xcolor}
\usepackage{dsfont}
\hypersetup{
    colorlinks=true,
    linkcolor=purple,
    citecolor=blue,
    filecolor=magenta,      
    urlcolor=cyan,
   }     
\usepackage{bm}

\usepackage[utf8]{inputenc} % allow utf-8 input
\usepackage[T1]{fontenc}    % use 8-bit T1 fonts
\usepackage{url}            % simple URL typesetting
\usepackage{amsfonts}       % blackboard math symbols
\usepackage{nicefrac}       % compact symbols for 1/2, etc.
\usepackage{amsmath,amssymb,amsthm}
\usepackage{bbm}
\usepackage{accents}

\usepackage{color}
\usepackage{mathtools}
\usepackage[shortlabels]{enumitem}
\usepackage{epstopdf}
\usepackage{tabulary}
\usepackage{cite}
\usepackage{chngcntr}
\usepackage{bm}
\usepackage{algorithm}
\usepackage{algpseudocode}

\newtheorem{theorem}{Theorem}
\newtheorem{lemma}{Lemma}
\newtheorem{proposition}{Proposition}

\newtheorem{definition}{Definition}
\newtheorem{remark}{Remark}

\title{Distributed Online Bandit Submodular Maximization with Bounded Sampling Violations}%
\author{%
Bin Du
\thanks{College of Automation Engineering, Nanjing University of Aeronautics~and Astronautics,~Nanjing 211106, China; \texttt{\{iniesdu, liuch, zdq.bpz\}@nuaa.edu.cn}}
\and
Chang Liu\footnotemark[1]
\and 
Dingqi Zhu\footnotemark[1]
\and
Lintao Ye%
\thanks{School of Artificial Intelligence and Automation at the Huazhong University of Science and Technology, Wuhan, China; \texttt{yelintao93@hust.edu.cn}}
\and
Dengfeng Sun%
\thanks{School of Aeronautics and Astronautics, Purdue University, West Lafayette, IN 47906, USA; \texttt{dsun@purdue.edu}}
}
\begin{document}
\maketitle

\begin{abstract}
We study distributed online submodular maximization under partition matroid constraints, in which multiple agents select a limited number of actions from their own subsets sequentially to maximize the cumulative value of a sequence of objective functions. We develop a unified algorithmic framework that accommodates full-information and bandit feedback models. For both feedback models, we prove that the proposed algorithms achieve sublinear $(1-1/e)$-regret guarantees, which are comparable to those achieved by existing centralized counterparts. Furthermore, to tackle the sampling violation issue caused by continuous relaxation and rounding, we develop a bounded stochastic pipage rounding scheme and show that the probability of sampling violation vanishes asymptotically. As a result, the cumulative sampling violation remains sublinear in $T$, which is further shown to be not improvable under certain conditions. Numerical results validate the theoretical findings in this paper.
\end{abstract}

\section{Introduction}

Submodular maximization is a fundamental problem in a broad range of applications such as sensor selection \cite{chamon2020approximate}, resource allocation \cite{paccagnan2021utility}, and task coordination with multi-agent systems \cite{zhou2022risk}. In its general form, the problem aims to maximize a set function subject to certain combinatorial constraints, such as cardinality or matroid constraints. Despite its applicability, this class of problems is known to be $\mathrm{NP}$-hard \cite{feige1998threshold} and hence its exact solutions are computationally intractable in general. As a consequence, a substantial body of research has focused on the development of efficient approximation algorithms with provable performance guarantees. A rich line of such work has established that constant-factor approximation ratios, e.g., $1/2$ or $1-1/e$, can be achieved via greedy-based methods or their continuous variants \cite{nemhauser1978analysis,fisher2009analysis,van2025performance}. Additionally, these guarantees are further shown to be optimal unless $\mathrm{P}=\mathrm{NP}$ \cite{feige1998threshold}. Such a result has provided a solid theoretical foundation for designing algorithms that achieve near-optimal performance guarantees while remaining computationally tractable.

Beyond the classical centralized problem, there has been a growing interest in studying distributed submodular maximization, see e.g., \cite{gharesifard2017distributed,vendrell2024optimality,mokhtari2018decentralized,du2022jacobi}, motivated by large-scale decision-making scenarios over multi-agent networks. One typical class of distributed problems arises when agents possess distributed action subsets, and feasibility is enforced through combinatorial constraints coupling all agents' decisions, e.g., \cite{robey2021optimal,du2022jacobi}. A canonical example is the partition matroid constraint \cite{rezazadeh2023distributed}, where the ground action set is partitioned over multiple agents and each one is allowed to select one or a few actions only from its local set. To solve the distributed submodular maximization under such constraints, the authors in \cite{robey2021optimal} propose a solution method that combines the classical continuous greedy algorithm (also known as the Frank-Wolfe algorithm) with the consensus-based coordination scheme. However, this method relies on the exact gradient evaluation of the multi-linear extension of the submodular objective function, leading to computational complexity that grows exponentially with the size of the action set. To address this issue, the authors in \cite{rezazadeh2023distributed} develop a distributed algorithm by following the same Frank-Wolfe framework but leveraging the empirical estimate of the gradient information. It is shown that the approximation ratio can be arbitrarily close to the optimal $1-1/e$ as the number of samples for the gradient estimation increases. Besides, our previous work \cite{du2022jacobi} proposes a stochastic gradient-based distributed algorithm which also considerably reduces the computational complexity of gradient evaluation. However, since the algorithm in \cite{du2022jacobi} is developed completely in a discrete domain, its approximation ratio is only guaranteed to be $1/2$ which is suboptimal compared to algorithms developed in the continuous domain.

Another rich line of work on submodular maximization has focused on the online setting, where a sequence of submodular objective functions is revealed over time and decisions must be made sequentially without access to future information. In this setting, the goal is to design an online algorithm that generates a sequence of decisions which maximizes the cumulative value of the time-varying objective functions. Specifically, the notion of cumulative $\alpha$-regret, e.g., $\alpha = 1/2$ or $1-1/e$, is utilized to characterize the performance of such online algorithms. If the cumulative $\alpha$-regret grows sublinearly with respect to the time horizon $T$, then the time-averaged performance of the algorithm asymptotically matches that of the $\alpha$-suboptimal solution chosen in hindsight. To this end, the authors in \cite{streeter2008online} propose the first online submodular maximization algorithm by using the \emph{meta-action} technique in the discrete domain. It is shown that the algorithm achieves $\mathcal{O}(\sqrt{T})$ $(1-1/e)$-regret under the cardinality constraint. This result is further extended in \cite{golovin2014online} to the problem with partition matroid constraints. By employing a standard lifting technique, i.e., lifting the discrete problem into the continuous domain followed by a proper rounding scheme, the same regret guarantee is established.

In fact, the meta-action idea introduced in \cite{streeter2008online} has also been applied to online continuous submodular maximization problems; see e.g., \cite{chen2018online,chen2018projection,zhang2019online}. In particular, the Meta-Frank-Wolfe algorithm is first proposed in \cite{chen2018online} and shown to reproduce the $\mathcal{O}(\sqrt{T})$ $(1-1/e)$-regret under general convex set constraints. However, a critical limitation of this approach is that it requires $\sqrt{T}$ evaluations of the exact gradient at each time $t$, which can be prohibitively expensive in practical scenarios. To alleviate this issue, the authors in \cite{chen2018projection} enhance the Meta-Frank-Wolfe algorithm by using stochastic gradient estimates and prove that the same regret result can be achieved under mild conditions. Nevertheless, the number of gradient evaluations remains on the order of $\sqrt{T}$, which may still be expensive for large-scale problems. Motivated by such a challenge, a one-shot version of the online algorithm, termed Mono-Frank-Wolfe, is developed in \cite{zhang2019online}. It is shown that this algorithm reduces the number of per-function gradient evaluations from $\sqrt{T}$ to one and achieves an $\mathcal{O}({T}^{4/5})$ $(1-1/e)$-regret. In addition, by applying the one-point gradient estimator \cite{hazan2016introduction}, the authors in \cite{zhang2019online} further adapt the algorithm development with a bandit feedback model and propose the first bandit submodular maximization method that achieves an $\mathcal{O}({T}^{8/9})$ $(1-1/e)$-regret. Additional related works can also be found in \cite{zhang2022stochastic,nie2023framework,wan2023bandit}. 

It is noteworthy that, by using the standard lifting technique, the aforementioned continuous algorithms can all be adopted to solve the discrete online submodular maximization problem. However, a critical limitation of such approaches is that they may require evaluating function values at sampled points that violate the original set constraints, such as (partition) matroid or cardinality constraints. Indeed, \cite{zhang2019online} shows that there is no proper rounding scheme that can simultaneously guarantee the unbiased gradient estimation and the feasibility of the resulting discrete samples. As a result, the responsive bandit algorithm is proposed in \cite{zhang2019online} which assigns zero rewards to the infeasible samples but still requires querying the function values at such violating points. In this paper, we aim to resolve the sampling violation issue.

More recently, a few attempts have been made to study submodular maximization problems in both distributed and online settings, see e.g., \cite{zhang2023communication,xu2023online,ye2025offline}. The authors in~\cite{zhang2023communication} consider the distributed online continuous submodular maximization problem in which each agent observes a local sequence of objective functions and all agents are subject to a common convex set constraint. Two distributed online algorithms are then presented, and their $(1-1/e)$-regrets are shown to be upper bounded by $\mathcal{O}({T}^{4/5})$ and $\mathcal{O}(\sqrt{T})$, respectively. To address the distributed setting with separate action subsets across agents, our previous work \cite{ye2025offline} develops an online distributed greedy algorithm in the purely discrete domain, which is shown to inherently achieve an $\mathcal{O}(\sqrt{T})$ $1/2$-regret. To the best of the authors' knowledge, no existing work addresses distributed online submodular maximization under partition matroid constraints while achieving the optimal $(1-1/e)$-regret. Closing this gap, as we detail below,~requires more than assembling existing techniques.

The contributions of this paper are summarized below:
\begin{enumerate}[(i)]
    \item We first formulate a distributed online submodular maximization problem under partition matroid constraints, and then propose a unified algorithmic framework that combines the lifting technique with the Meta-Frank-Wolfe algorithm, accommodating both full-information and bandit feedback models. To the best of the authors' knowledge, this is the first framework to achieve the optimal $(1-1/e)$-regret for this problem.
    \item For the full-information and bandit models, we~establish the $\mathcal{O}({T}^{4/5})$ and $\mathcal{O}({T}^{8/9})$ $(1-1/e)$-regret bounds, respectively, which match the known centralized rates while operating in a fully distributed manner over the network.
    \item To address the sampling violations inherent to online rounding under set constraints, we introduce \emph{bounded stochastic pipage rounding} (\texttt{B-SPR}), which generalizes pipage rounding to arbitrary fractional bounds while retaining its sum and monotonicity guarantees. Embedding~\texttt{B-SPR}~into a progressively~bounded~scheme~(\texttt{PB-SPR})~then drives the iterates toward a matroid vertex. We prove that per-step violation probability vanishes as $\mathcal{O}(T^{-1/9})$, so the cumulative violation grows only sublinearly as $\mathcal{O}(T^{7/9})$. A matching $\Omega(T^{7/9})$ lower bound certifies that this rate is not improvable under certain conditions.
\end{enumerate}

{\bf Notations:} The sets of integers and real numbers are denoted as $\mathbb{Z}$ and $\mathbb{R}$, respectively. For any given integer $n\in\mathbb{Z}_{\ge1}$, let $[n]:=\{1,2,\cdots,n\}$ and $\mathbf{1}_n$ denote the~$n$-dimensional one vector. For a set $S$, let $|S|$ be the cardinality of the set. For two vectors $\mathbf{x},\mathbf{y} \in \mathbb{R}^n$, the inequality $\mathbf{x} \preceq \mathbf{y}$ is defined element-wise. $\mathbbm{1}\{\cdot\}$ denotes the indicator function.

\section{Problem Statement and Preliminaries}

\subsection{Distributed Online Submodular Maximization}
Consider a network of $I$ agents, where each agent $i\in [I]$ maintains a local action subset $\mathcal{A}_i \subseteq \mathcal{A} := \cup_{i=1}^I\mathcal{A}_i$ with $|\mathcal{A}_i| = N_i$ and $|\mathcal{A}| = N$. The objective of each agent is to select at most $\kappa_i \in \mathbb{Z}_{\ge 1}$ actions from $\mathcal{A}_i$, such that a sequence of time-varying or even \emph{adversarial} objective functions $f^1(\mathcal{S})$, $f^2(\mathcal{S}), \cdots, f^T(\mathcal{S})$ is cooperatively maximized over the time horizon $T \in \mathbb{Z}_{\ge 1}$. Such a problem can be formulated as the distributed maximization,
\begin{subequations}\label{submodularMax}
    \begin{align}
        \text{maximize}\;\;\, &f^t(\mathcal{S})\label{objFunc}\\
        \text{subject to}\;\;\,         & \mathcal{S} \in \mathcal{I} :=\{ \mathcal{S} \,\big\vert\, |\mathcal{S} \cap \mathcal{A}_i |\le \kappa_i, \;i \in [I]  \}.\label{constr}
    \end{align}
\end{subequations}
Notice that, by assuming $\mathcal{A}_i \cap \mathcal{A}_j =\emptyset,\,\forall i\neq j$, the set $\mathcal{I}$ in~\eqref{constr} is known as a \emph{partition matroid} \cite{robey2021optimal}. Moreover, we assume that each function $f^t(\cdot): 2^\mathcal{A} \to \mathbb{R}_{\ge0}$ is \emph{monotone}, \emph{submodular}, and \emph{normalized}, whose definitions are provided as follows.

\begin{definition}[Monotonicity]
    A set function $f(\cdot)$ is monotone non-decreasing if, for $\forall \mathcal{S}_1 \subseteq \mathcal{S}_2 \subseteq \mathcal{A}$, it holds that $f(\mathcal{S}_1) \le f(\mathcal{S}_2)$.
\end{definition}

\begin{definition}[Submodularity]
   A set function $f(\cdot)$ is submodular if, for $\forall \mathcal{S}_1 \subseteq \mathcal{S}_2 \subseteq \mathcal{A}$ and $\forall a \in \mathcal{A} \setminus \mathcal{S}_2$, it holds that $f(\{a\} \cup \mathcal{S}_1) - f(\mathcal{S}_1) \ge f(\{a\} \cup \mathcal{S}_2) - f(\mathcal{S}_2)$.
\end{definition}

In addition to the properties of monotonicity and submodularity, it is often useful if the set function is also \emph{normalized}, i.e., $f(\emptyset) = 0$. Without loss of generality, this property can be enforced by adding an appropriate constant to the function \cite{chamon2020approximate}.

In this work, we focus on an online setting of the submodular maximization problem where the agents have no a priori knowledge of the objective functions. Specifically, at each time $t \in [T]$, each agent selects a subset of actions $\mathcal{S}^t_i \subseteq \mathcal{A}_i$ such that $|\mathcal{S}^t_i| \le \kappa_i$. After all agents commit to their decisions $\mathcal{S}^t_i$'s, a potential adversary selects an objective function $f^t(\cdot)$ and certain information regarding the function is revealed to the agents. In the full-information feedback model, the agents have access to the full knowledge of $f^t(\cdot)$ and thus can evaluate the function value for any selections of actions. By contrast, in the bandit feedback model, only the function value of selected decisions, i.e., $f^t(\cup_{i=1}^I \mathcal{S}_i^t)$, can be observed, leading to a more challenging setup for the decision-making of agents. Although the agents have no knowledge of $f^t(\cdot)$ when selecting $\mathcal{S}_i^t$, they may exploit previously revealed functions to guide their decisions. The problem where the sequence $\{f^t(\cdot)\}_{1\le t\le T}$ is arbitrary is often referred to as an \emph{adversarial online setting} \cite{chen2018projection}. In this setting, the goal of agents is to cooperatively minimize the \emph{adversarial cumulative regret} given by
\begin{align}\label{eq:regret}
    \mathcal{R}_T := (1-1/e)\cdot \max_{\mathcal{S} \in \mathcal{I}} \;\sum_{t=1}^Tf^t(\mathcal{S}) - \sum_{t=1}^T f^t(\cup_{i=1}^I \mathcal{S}_i^t).
\end{align}
The notion of regret in \eqref{eq:regret} is defined with respect to a joint reward $f^t(\cup_{i=1}^I \mathcal{S}_i^t)$ received by all agents. Such a reward model has been widely used by many works on the distributed submodular maximization problem, see e.g., \cite{rezazadeh2023distributed,gharesifard2017distributed,du2022jacobi}. In addition, the factor $1\!-\!1/e$ in \eqref{eq:regret} is due to the $\mathrm{NP}$-hardness of the standard submodular maximization problem and therefore only approximate solutions can be expected in general by algorithms running in reasonable time \cite{feige1998threshold}.

Moreover, we consider that the communication topology of the $I$-agent network is modeled by a time-invariant undirected graph $\mathcal{G}:=\{\mathcal{N}, \mathcal{E}\}$ where $\mathcal{N} := [I]$ and $\mathcal{E} \subset \mathcal{N} \times \mathcal{N}$ denote the sets of nodes and edges, respectively. The agent $i$ receives information from the agent $j$ if $(i,j) \in \mathcal{E}$. Particularly, we let $\mathcal{N}_i:=\{j \in [I] \;\vert\; (i,j) \in \mathcal{E}\}$ represent the (in-)neighborhood of the agent $i$, i.e., the set of agents that can send information to the $i$-th agent directly. In addition, it is also assumed that the undirected graph $\mathcal{G}$ is \emph{connected}, i.e., every agent $i$ can receive information from all others through a path. We denote by $d(\mathcal{G})$ the diameter of $\mathcal{G}$, i.e., the greatest shortest-path distance over all pairs of agents $(i,j)$. Due to the connectedness of $\mathcal{G}$, it is easy to verify that $d(\mathcal{G}) \le I$.

\subsection{Multi-Linear Extension}

To achieve the tightest approximation bound for solving the classical submodular maximization problems, a suite of well-studied methods uses the notion of multi-linear extension \cite{vondrak2008optimal}. To provide a formal definition of the multi-linear extension of a submodular set function,  let us assume without loss of any generality that the ground set is given by $\mathcal{A}=\{1,2,\cdots, N\}$ and the local subsets $\mathcal{A}_i$'s contain consecutive integers and are ordered such that, if $ \mathcal{A}_i = \{p,p+1,\cdots, q\}$, then one can have $p-1 \in \mathcal{A}_{i-1}$ and $q+1 \in \mathcal{A}_{i+1}$. Associated with $\mathcal{A}$, we define a continuous variable $\mathbf{x} = [x(1),x(2),\cdots, x(N)]^\top \in [0,1]^N$. Further, we let the $i$-th block of the vector $\mathbf{x}$ be $[\mathbf{x}]_i \in [0,1]^{N_i}$ whose entries are associated with the local subset $\mathcal{A}_i$. Using the notations defined above, we can express the multi-linear extension of each function $f^t(\cdot)$ as
\begin{align}
  F^t(\mathbf{x}):=\sum_{\mathcal{S} \subseteq \mathcal{A}} f^t(\mathcal{S}) \cdot\prod_{n \in \mathcal{S}}x(n)\cdot\prod_{n \notin \mathcal{S}}\big(1-x(n)\big).
\end{align}
Corresponding to the partition matroid constraint \eqref{submodularMax}, we introduce the matroid polytope $\mathcal{M} \subseteq [0,1]^N$ where
\begin{align}\label{eq:def-M}
  \mathcal{M} := \{ \mathbf{x} \in [0,1]^N \,\vert\, \mathbf{1}_{N_i}^\top[\mathbf{x}]_i \le \kappa_i, \; \forall i \in [I]\}.
\end{align}
Similarly, we define the local polytope $\mathcal{M}_i \subseteq [0,1]^{N_i}$ for each agent $i\in[I]$ as
\begin{align}\label{eq:def-Mi}
  \mathcal{M}_i := \{ \mathbf{x} \in [0,1]^{N_i} \,\vert\, \mathbf{1}_{N_i}^\top \mathbf{x} \le \kappa_i \}.
\end{align}
It is straightforward to have that $\mathcal{M} = \prod_{i \in[I]} \mathcal{M}_i$.

The above notion of multi-linear extension transfers the classical submodular maximization problem from the discrete domain to the continuous domain. Precisely, the solution to $\max_{\mathcal{S} \in \mathcal{I}} f(\mathcal{S})$ is equivalent to solving the following continuous optimization \cite{vondrak2010submodularity},
\begin{align}\label{submodularMax-continuous}
    \mathop {\text{maximize} }\limits_{\mathbf{x} \in \mathcal{M}}  \;\;F(\mathbf{x}).
\end{align}
In addition, even if the suboptimal solution $\widetilde{\mathbf{x}}$ to \eqref{submodularMax-continuous} is obtained, one can adopt various rounding schemes to recover an approximate discrete solution $\widetilde{\mathcal{S}}$ that ensures \mbox{$f(\widetilde{\mathcal{S}}) \ge F(\widetilde{\mathbf{x}})$}. Instances of such \emph{lossless} methods include pipage rounding \cite{calinescu2011maximizing}, central rounding \cite{chekuri2011submodular}, etc.

\subsection{Frank-Wolfe Algorithm and Tracking the Expert}\label{subsec:FW-expert}

To solve the maximization problem \eqref{submodularMax-continuous} in the continuous domain, a well-studied method is Frank-Wolfe \cite{frank1956algorithm}, which performs $\mathbf{x}^{k+1} = \mathbf{x}^k + \eta^k \mathbf{v}^k$ where $\mathbf{x}^0$ is initialized as $\mathbf{0}$ and $\eta^k \in \mathbb{R}_{>0}$ is the step-size. Specifically, the updating direction $\mathbf{v}^k \in \mathbb{R}^N$ can be obtained by maximizing a linear function under constraints, i.e., 
\begin{align}\label{frank-wolfe}
    \mathbf{v}^k = \arg\max_{\mathbf{v} \in \mathcal{M}}\;\langle \mathbf{v}, \nabla F(\mathbf{x}^k) \rangle.
\end{align}
It has been shown in \cite{bian2017guaranteed} that the iterates $\mathbf{x}^k$ produced by Frank-Wolfe can be arbitrarily close to the $(1-1/e)$ approximate solution, i.e., $F(\mathbf{x}^k) \hspace{-0.5pt}\to\hspace{-0.5pt} (1 - 1/e)F(\mathbf{x}^\star)$ as \mbox{$k \hspace{-0.5pt}\to\hspace{-0.5pt} \infty$} where $\mathbf{x}^\star$ denotes the optimal solution to \eqref{submodularMax-continuous}.

Although the Frank-Wolfe algorithm provides a promising solution method to the continuous optimization \eqref{submodularMax-continuous}, a critical drawback is that it heavily relies on the gradient $\nabla F(\cdot)$ during the iterations. However, under the \emph{adversarial online setting}, one can only have access to the knowledge of function $F^t(\cdot)$ after committing to the decisions at this round. Therefore, to perform the Frank-Wolfe update, it is required to first obtain a sequence of approximate solutions to the maximization problem \eqref{frank-wolfe} without knowing the gradient information \cite{zhang2019online}. 

To resolve the above issue, a natural idea is to use the technique of \emph{tracking the expert} which has been widely adopted in solving online optimization problems. The expert problem involves an agent selecting actions $a_t \in \mathcal{A}$ to maximize the cumulative reward $\sum_{t =1}^T r_t(a_t)$ over a finite time horizon $T \in \mathbb{Z}_{\ge 1}$. The main challenge of such a problem lies in the fact that the time-varying reward $r_t(\cdot)$ can be unknown to the agent before committing to the action $a_t$, which is analogous to our adversarial online setting. Consequently, to solve this problem, the agent needs to track an expert's decision, i.e., the action $a_T^\star \in$ $ \arg\max_{a \in \mathcal{A}}\, \sum_{t=1}^Tr_t(a)$ leading to the highest cumulative reward in hindsight. Indeed, there have been various approaches proposed in the literature, e.g., Regularized Follow the Leader (RFTL) \cite{hazan2016introduction}, Follow the Perturbed Leader (FPL) \cite{cohen2015following}, and the online gradient descent \cite{zinkevich2003online}, which all manage to track the expert. In this work, we specifically adopt the RFTL algorithm, whose performance is guaranteed by the following result (see Theorem 5.1 in \cite{hazan2016introduction}): there exists a constant $C_0 >0$ such that the generated sequence of actions $\{a_t\}_{1\le t\le T}$ has
\begin{align}\label{ineq:expert-regret}
\max_{a\in\mathcal{A}} \;\sum_{t=1}^T r_t(a) - \sum_{t=1}^T \mathbb{E} [r_t(a_t)] \le C_0 \sqrt{T},
\end{align}
where the expectation is taken with respect to the randomness of the generation of $a_t$.

\section{Algorithm with Full-Information Feedback}\label{sec:full-info}

In this section, we first introduce a vanilla version of the distributed online algorithm under the full-information feedback setting, which also establishes analytical~tools for the subsequent extensions.

\subsection{Distributed Online Algorithm: A Vanilla Version}\label{subsec:algo-full-info}

The distributed online algorithm with full-information feedback is outlined in Algorithm \ref{algo:full-info-feedback}, which consists of two phases. In the distributed Frank-Wolfe phase, each~agent  performs a local block-coordinate Frank-Wolfe~update (see Line~6) and then reaches a maximum consensus~on the local variables $\mathbf{x}_i^q(k)$'s with its neighbors (see Line~7), thereby building coordination over the network. The~online learning phase, inspired by \cite{zhang2019online}, requires only a one-shot gradient per function. Following a random permutation $\{t_i^{q,1}, \cdots, t_i^{q,K}\}$~of the indices within block $q\in[Q]$ (see Line~9), each agent applies~the lossless rounding on $[\mathbf{x}_{i}^q(K+1)]_i$ to obtain its discrete decision $\mathcal{S}_i^t$, which is executed for all time-steps within the block $q$ (see Line 11). With an appropriate lossless rounding, e.g., the stochastic pipage rounding \cite{calinescu2011maximizing}, the joint action $\mathcal{S}^t = \cup_{i\in [I]} \mathcal{S}_i^t$ is always feasible, i.e., $\mathcal{S}^t \in \mathcal{I}$ for $\forall t \in [T]$, and the global reward $f^t(\mathcal{S}^t)$ is received by all agents. The gradients $\nabla F^{t_i^{q,k}}\hspace{-1pt}(\mathbf{x}_i^q(k))$ evaluated along the permutation mimic the stochastic gradient of the block-averaged function $\bar{F}(\cdot) = 1/K\cdot\sum_{k=1}^K F^{(q-1)K+k} (\cdot)$ (see Line~13), and an averaging technique \cite{mokhtari2020stochastic} is adopted to reduce the variance of the gradient estimate $\mathbf{d}_i^q(k)$ (see Line~15).

\begin{algorithm}
\textbf{Input:} Horizon length $T$; block size $K$; number of blocks $Q = T/K$; step-size $\rho_k$; communication topology $\mathcal{G}$; local constraint $\mathcal{M}_i$.
\caption{-- Distributed online algorithm (full information feedback)}\label{algo:full-info-feedback}
\begin{algorithmic}[1]
\State Each agent $i\in[I]$ initializes {experts} $\mathcal{E}_i(k), \, k \in [K]$  for maximizing linear cost over $\mathcal{M}_i$.
\For{$q=1,2,\cdots,Q$}

\hspace{-20pt}\# \texttt{Distributed Frank-Wolfe Phase}
\State Initialize $\mathbf{x}_i^q (1) = \mathbf{0}$ and $\mathbf{d}_i^q(0) = \mathbf{0}$.
\For{$k = 1,2,\cdots,K$}
\State Let $\mathbf{v}_i^q(k) \in \mathcal{M}_i$ be the output of $\mathcal{E}_i(k)$;
\State Update 
$[{\mathbf{x}}_i^q(k+1)]_i = [\mathbf{x}_i^q (k)]_i + \mathbf{v}_i^q(k)/K;$
\State Receive $[{\mathbf{x}}_j^q(k+1)]_{i'}$ from neighbors $j \in \mathcal{N}_i$, and update
$$[{\mathbf{x}}_i^q(k+1)]_{i'} = \max_{j \in \mathcal{N}_i \cup \{i\}}[{\mathbf{x}}_j^q(k+1)]_{i'}, \quad \forall i'\neq i.$$
\EndFor

\hspace{-20pt}\# \texttt{Online Learning Phase}
\State Let $\{t_i^{q,1}, t_i^{q,2}, \cdots, t_i^{q,K}\}$ be a random permutation of $\{(q-1)K+1, (q-1)K+2, \cdots, qK\}$.
\For{$t = (q-1)K+1, (q-1)K+2, \cdots, qK$}
\State Generate and execute the discrete decision $\mathcal{S}_i^t$ by applying the \textbf{\emph{lossless rounding}} on $\mathbf{s}_i(t) := [\mathbf{x}_{i}^q(K+1)]_i$, and receive a global reward $f^t(\mathcal{S}^t)$ where $\mathcal{S}^t:= \cup_{i\in [I]} \mathcal{S}_i^t$;
\State {Observe $f^t(\cdot)$ and its multi-linear extension function $F^t(\cdot)$;}
\State {Find $k \in [K]$ such that $t_i^{q,k} = t$, and compute the gradient ${\nabla} {F^{t}}(\mathbf{x}_i^q(k))$.}
\EndFor
\State For $k \in [K]$, let $\mathbf{d}_i^q(k) = (1-\rho_k)\mathbf{d}_i^q(k-1) +$ $ \rho_k \nabla F^{t_i^{q,k}}(\mathbf{x}_i^q(k))$, and feed $ [\mathbf{d}_i^q(k)]_i$ back to $\mathcal{E}_i(k)$.
\EndFor
\end{algorithmic}
\end{algorithm}

We note that the required gradient $\nabla F^{t_i^{q,k}}\hspace{-2pt}(\mathbf{x}_i^q(k))$ can~be conceptually obtained under the full-information model, since each agent has full knowledge of $f^t(\cdot)$ (or $F^t(\cdot)$). Nevertheless, by the definition of multi-linear extension, evaluating the exact gradient is combinatorial in nature and thus computationally intractable even for moderately large problems. A standard remedy is Monte Carlo sampling \cite{rezazadeh2023distributed}, which still requires full-information feedback; we relax this requirement in the next section.

\subsection{Regret Analysis}\label{subsec:reget-analysis-full-info}

We first characterize the consensus performance of~the distributed Frank-Wolfe phase in the following lemma.

\begin{lemma}\label{lemma:consensus-full-info}
 Let $\bar{\mathbf{x}}^q(k):= \big[[\mathbf{x}_1^q(k)]_1^\top,  \cdots\!, [\mathbf{x}_I^q(k)]_I^\top\big]^\top$~where $\mathbf{x}_i^q(k)$ is generated by Algorithm \ref{algo:full-info-feedback}, the following statements hold for $\forall q \in [Q]$ and $k \in [K+1]$:
 \begin{enumerate}[i)]
     \item $\bar{\mathbf{x}}^q(k) \in \mathcal{M}$ and $\mathbf{x}_i^q(k) \in \mathcal{M}$, $\forall i\in[I]$;
     \item $\|\bar{\mathbf{x}}^q(k) - \mathbf{x}_i^q(k)\| \hspace{-1pt}\le\hspace{-1pt} \kappa d(\mathcal{G})/K$, $\forall i\hspace{-1pt}\in\hspace{-1pt}[I]$, $\kappa = \sum_{i=1}^I \kappa_i$.
 \end{enumerate}
\end{lemma}
\begin{proof}
    See Appendix I-A.
\end{proof}

With the aid of Lemma~\ref{lemma:consensus-full-info}, we are now ready to provide the regret bound of Algorithm \ref{algo:full-info-feedback} in the following theorem. Note that we here assume, for simplicity of presentation, that $K$ is an even number. This condition can be relaxed by more involved algebraic calculations.

\begin{theorem}\label{thm:regret-full-info}
    Suppose that the function $f^t(\cdot),\, \forall t \in [T]$ is monotone, submodular, and normalized. Let $K = T^{3/5}$ and $\rho_k$ be selected by $\rho_k = 2/(k+3)^{2/3}$ for $1 \le k \le K/2+1$ and $\rho_k = 1.5/(K-k+2)^{2/3}$ for $K/2+2 \le k \le K$, then the cumulative regret of Algorithm \ref{algo:full-info-feedback} satisfies
    \begin{equation}
     \begin{aligned}
        \mathbb{E}[\mathcal{R}_T] &\le \big(C_1 + C_0I+ D^2\big) \cdot T^{4/5} + \big(L_2D^2/2 + L_2\kappa d(\mathcal{G})\sum_{i=1}^I D_i\big)\cdot T^{2/5},
    \end{aligned}   
    \end{equation}
    where $C_1 = 4L_1^2+32\cdot(2L_1+DL_2)^2$, $D_i =\sqrt{ \min\{2\kappa_i, N_i\}}$, $D = \sqrt{\sum_{i=1}^I D_i^2}$, $L_1 = 2 \sqrt{N}f^\text{max}$ and $L_2 = 4Nf^\text{max}$ with $f^\text{max}$ satisfying $ \max_{\mathcal{S} \in \mathcal{I}} f^t (\mathcal{S}) \le  f^\text{max}, \forall t\in[T]$.
\end{theorem}
\begin{proof}
    See Appendix I-B.
\end{proof}

\section{Algorithm with Bandit Feedback}

In this section, we adapt the algorithm of Sec. \ref{sec:full-info} to the bandit feedback model and establish its sublinear regret.

\subsection{Distributed Online Bandit Algorithm}

Considering that only bandit feedback is available to~the agents, a crucial step for the algorithm development here is to provide a proper estimation of the gradient as required by the previous algorithm. Towards this end, we use the classical one-point gradient estimator \cite{zhang2019online,hazan2016introduction}, which is briefly reviewed as follows. For a function $F(\cdot): \mathbb{R}^{N} \to \mathbb{R}$, its $\delta$-smoothed approximation can be defined as $\widehat{F}_\delta(\mathbf{x}):= \mathbb{E}_{\mathbf{v} \sim \text{Uni}(\mathbb{B}^N)}\big[F(\mathbf{x} + \delta\mathbf{v})\big]$ where $\mathbf{v} \in \mathbb{R}^N$ is a sample uniformly drawn from the unit ball $\mathbb{B}^N$. Then, the gradient of the approximation function is shown to be $\nabla \widehat{F}_\delta(\mathbf{x}):= N/\delta\cdot \mathbb{E}_{\mathbf{u} \sim \text{Uni}(\mathbb{S}^{N-1})}\big[F(\mathbf{x} + \delta\mathbf{u})\mathbf{u}\big]$, where $\mathbb{S}^{N-1}$ denotes the unit sphere (see Lemma 6.4 in \cite{hazan2016introduction}). Therefore, the value of $N/\delta\cdot F(\mathbf{x} + \delta\mathbf{u})\mathbf{u}$ can be applied as a one-point unbiased stochastic estimate of the gradient $\nabla F(\mathbf{x})$. It should be noted that, to ensure the feasibility of the sampled point, i.e., $\mathbf{x} + \delta\mathbf{u} \in \mathcal{M}$, the gradient can only be estimated when $\mathbf{x}$ is taken from the interior of $\mathcal{M}$. Specifically, a notion of $\delta$-interior of $\mathcal{M}$ is defined in \cite{zhang2019online}, i.e., $\widetilde{\mathcal{M}} := \alpha \mathcal{M} + \delta \cdot \mathbf{1}_N$. It is shown (see~Lemma~1 in \cite{zhang2019online}) that $\mathbf{x} + \delta \mathbf{u} \in \mathcal{M}$ holds for $\forall \mathbf{x} \in \widetilde{\mathcal{M}}$ and $\mathbf{u} \in \mathbb{S}^{N-1}$, if $\alpha$ is chosen as $\alpha = 1- (\sqrt{N}+1)\delta/\gamma$ where $\gamma>0$ denotes the radius of a ball whose positive orthant is contained in $\mathcal{M}$. Thus, to validate the stochastic~gradient estimation, we next restrict $\mathbf{x}$ to be in the $\delta$-interior set $\widetilde{\mathcal{M}}$.

\begin{algorithm}[ht]
\textbf{Input:} Horizon length $T$; block size $L$; number of blocks $Q = T/L$; number of gradient estimation steps per block \mbox{$K \le L$}; step-sizes $\rho_k$; communication graph $\mathcal{G}$; constraint $\mathcal{M}_i$; and parameters $\delta$, $\alpha$.
\caption{-- Distributed online algorithm (bandit feedback)}\label{algo:bandit-feedback}
\begin{algorithmic}[1]
\State Each agent $i\in[I]$ initializes {experts} $\mathcal{E}_i(k), \, k \in [K]$  for maximizing linear cost over the local constraint $\widetilde{\mathcal{M}}_i:= \alpha \mathcal{M}_i + \delta \cdot \mathbf{1}_{N_i}$.
\For{$q=1,2,\cdots,Q$}

\hspace{-20pt}\# \texttt{Distributed Frank-Wolfe Phase}
\State Initialize $\mathbf{x}_i^q (1) = \delta\cdot\mathbf{1}_{N}$ and $\mathbf{d}_i^q(0) = \mathbf{0}$.
\For{$k = 1,2,\cdots,K$}
\State Let $\mathbf{v}_i^q(k) \in \widetilde{\mathcal{M}}_i$ be the output of $\mathcal{E}_i(k)$;
\State Update 
$$[{\mathbf{x}}_i^q(k+1)]_i = [\mathbf{x}_i^q (k)]_i +  \big(\mathbf{v}_i^q(k) - \delta\cdot \mathbf{1}_{N_i}\big)/K;$$
\State Receive $[{\mathbf{x}}_j^q(k+1)]_{i'}$ from neighbors $j \in \mathcal{N}_i$, and update
$$[{\mathbf{x}}_i^q(k+1)]_{i'} = \max_{j \in \mathcal{N}_i \cup \{i\}}[{\mathbf{x}}_j^q(k+1)]_{i'}, \quad \forall i'\neq i.\vspace{-7pt}$$
\EndFor

\hspace{-20pt}\# \texttt{Online Learning Phase}
\State Let $\{t_i^{q,1}, t_i^{q,2}, \cdots, t_i^{q,L}\}$ be a random permutation of $\{(q-1)L+1, (q-1)L+2, \cdots, qL\}$.
\For{$t = (q-1)L+1, (q-1)L+2, \cdots, qL$}
\State Generate and execute the discrete decision $\mathcal{S}_i^t$ by applying the \textbf{\emph{lossless rounding}} on $\mathbf{s}_i(t) := [\mathbf{x}_{i}^q(K+1)]_i$, and receive a global reward $f^t(\mathcal{S}^t)$ where $\mathcal{S}^t:= \cup_{i\in [I]} \mathcal{S}_i^t$.
\If{$t \in \{t_i^{q,1}, t_i^{q,2}, \cdots, t_i^{q,K}\}$}
\State  {Find $k \in [K]$ such that $t_i^{q,k} = t$, and generate discrete variable $\mathcal{X}^q_i(k)$ by applying \textbf{\emph{random}} \textbf{\emph{rounding}} on $\mathbf{x}_{i}^q(k) + \delta\mathbf{u}_{i}^q(k)$ where $\mathbf{u}_{i}^q(k)$ is sampled uniformly from $\mathbb{S}^{N-1}$, i.e., $\mathbf{u}_{i}^q(k) \sim \text{Uni}(\mathbb{S}^{N-1})$;
\State Receive $R_i(k) = f^t(\mathcal{X}^q_i(k))$}. 
\EndIf
\EndFor
\State For $k \in [K]$, let $\mathbf{d}_i^q(k) = (1-\rho_k)\mathbf{d}_i^q(k-1) +$ $ \rho_k  N_i R_i(k)\mathbf{u}_i^q(k)/\delta$, and feed $[\mathbf{d}_i^q(k)]_i$ to $\mathcal{E}_i(k)$.
\EndFor
\end{algorithmic}
\end{algorithm}

Applying the above one-point gradient estimator, the distributed online algorithm with bandit feedback is outlined in Algorithm \ref{algo:bandit-feedback}. The distributed Frank-Wolfe phase follows that of Algorithm \ref{algo:full-info-feedback}, except that the local variable is restricted to the $\delta$-interior $\widetilde{\mathcal{M}}_i := \alpha \mathcal{M}_i + \delta \cdot \mathbf{1}_{N_i}$ of $\mathcal{M}_i$. In the online learning phase, each block $q$ has length $L$, and while the global reward $f^t(\mathcal{S}^t)$ is received at every time-step (see Line~11), the one-point gradient estimation is performed only at $K$ (out of $L$) steps (see Lines 13 and 14). The discrete variable $\mathcal{X}^q_i(k)$ is generated by applying the \emph{random rounding} on the sample $\mathbf{x}_{i}^q(k) + \delta\mathbf{u}_{i}^q(k)$, which ensures the unbiasedness $\mathbb{E}[f^t(\mathcal{X}^q_i(k))]=$ $F(\mathbf{x}_{i}^q(k) + \delta\mathbf{u}_{i}^q(k))$. As before, an averaging technique is applied to $\mathbf{d}_i^q(k)$ for variance reduction (see Line~17).

\subsection{Regret Analysis}
Similar to the analysis of Algorithm \ref{algo:full-info-feedback}, we here first show the consensus result obtained by the distributed Frank-Wolfe phase in Lemma \ref{lemma:consensus-bandit}, and then provide the regret bound of Algorithm \ref{algo:bandit-feedback} in Theorem \ref{thm:regret-bandit}. Note that the following Lemma \ref{lemma:consensus-bandit} is analogous to Lemma \ref{lemma:consensus-full-info} in Sec. \ref{subsec:reget-analysis-full-info}, while the only difference is that the iterates $\mathbf{x}_i^q(k)$ are shown to be in the $\delta$-interior set $\widetilde{\mathcal{M}}:=\alpha \mathcal{M} + \delta\cdot\mathbf{1}_N$.
\begin{lemma}\label{lemma:consensus-bandit}
     Let $\bar{\mathbf{x}}^q(k):= \big[[\mathbf{x}_1^q(k)]_1^\top,  \cdots\!, [\mathbf{x}_I^q(k)]_I^\top\big]^\top$~where $\mathbf{x}_i^q(k)$ is generated by Algorithm \ref{algo:bandit-feedback}, the following statements hold for $\forall q \in [Q]$ and $k \in [K+1]$:
 \begin{enumerate}[i)]
     \item $\bar{\mathbf{x}}^q(k) \in \widetilde{\mathcal{M}}$ and $\mathbf{x}_i^q(k) \in \widetilde{\mathcal{M}}$, $\forall i\in[I]$;
     \item $\|\bar{\mathbf{x}}^q(k) - \mathbf{x}_i^q(k)\| \hspace{-1pt}\le\hspace{-1pt} \alpha \kappa d(\mathcal{G})/K$, $\forall i\hspace{-1pt}\in\hspace{-1pt}[I]$.
 \end{enumerate}
\end{lemma}
\begin{proof}
    See Appendix II-A.
\end{proof}

\begin{theorem}\label{thm:regret-bandit}
    Suppose that the function $f^t(\cdot), \forall t \in [T]$ is monotone, submodular, and normalized. Let $L \hspace{-1pt}=\hspace{-1pt} T^{7/9}$, $K\hspace{-3pt}=\hspace{-1pt}T^{2/3}$, $\alpha\hspace{-1pt}=\hspace{-1pt}1-(\sqrt{N}+1)\delta/\gamma$, $\delta = T^{-1/9}\gamma/(\sqrt{N}+2)$, $\gamma = \min_{i \in [I]}\big\{\min\{1, \kappa_i / \sqrt{N_i}\}\big\}$,
    and $\rho_k = 2/(k+3)^{2/3}$ for $\forall k \in [K]$, then the cumulative regret of Algorithm \ref{algo:bandit-feedback} satisfies
    \begin{equation}\label{ineq:regret-bound-bandit}
     \begin{aligned}
        \mathbb{E}[\mathcal{R}_T] &\le\Big((1-1/e)L_1C_\mathcal{M}C_\delta + (2-1/e)L_1 C_\delta  + 3/4\cdot C_2/C_\delta + C_0 I + 3/4\cdot {D}^2/C_\delta\Big)\cdot T^{8/9}\\
          & \quad+3/4\cdot C_3C_\delta\cdot  T^{2/3} + L_1 \kappa d(\mathcal{G})N \sum_{i=1}^I {D}_i/C_\delta \cdot T^{4/9} +L_2{D}^2 /2 \cdot T^{1/3},
    \end{aligned}   
    \end{equation}
    where $C_\mathcal{M} = (\sqrt{\kappa}/\gamma + 1)\sqrt{N} + \sqrt{\kappa}/\gamma$, $C_\delta = \gamma/(\sqrt{N}+2)$, $C_2 = 2\sqrt[3]{16} L_1^2N^2 $, and $C_3 = \sqrt[3]{16}\big(2L_1^2 + (2L_1 + 3{D}L_2)^2\big)$.
\end{theorem}
\begin{proof}
    See Appendix II-B.
\end{proof}

Theorem \ref{thm:regret-bandit} shows that Algorithm \ref{algo:bandit-feedback} also achieves sublinear expected regret. Compared to~Theorem~\ref{thm:regret-full-info}, however, the rate degrades to $\mathcal{O}(T^{8/9})$ as only bandit feedback is available. Before the end of this section, we would~also like to add a few remarks regarding the algorithm and~its regret analysis.

\begin{remark}
    Two pieces of information are received as feedback in Algorithm \ref{algo:bandit-feedback} during the iteration. While the~reward $f^t(\mathcal{S}^t)$ is observed at each time $t \in [T]$, the function value $f^{t_i^{q,k}}\!\big(\mathcal{X}^q_i(k)\big)$ is only received at a subset of the time horizon, i.e., for $\forall k \in [K], q \in[Q]$. However, it should be noted that the observations of $f^t(\mathcal{S}^t)$ are used solely for counting the cumulative reward/regret, but not involved in the algorithmic updates. Therefore, we would like to argue that our algorithm still operates under the bandit feedback setting, in which only the function value at one single queried point is revealed at each iteration. In contrast, the regret analysis in \cite{zhang2019online} counts $f^{t_i^{q,k}}\!\big(\mathcal{X}^q_i(k)\big)$, rather than $f^t(\mathcal{S}^t)$, as the reward received by each agent at time $t = t_i^{q,k}$. However, under the conditions in Theorem \ref{thm:regret-bandit}, the two formulations are inherently equivalent, since the resulting difference in cumulative regret is naturally bounded by a sublinear term with respect to $T$, i.e., $\sum_{q=1}^Q \sum_{k=1}^K |f^{t_i^{q,k}}(\mathcal{S}^{t_i^{q,k}}) - f^{t_i^{q,k}}\!\big(\mathcal{X}^q_i(k)\big)| \le \mathcal{O}(T^{8/9})$.
\end{remark}

\begin{remark}\label{remark:violating-sample}
    Though it has been ensured by Lemma \ref{lemma:consensus-bandit} and the definition of $\widetilde{\mathcal{M}}$ that $ \mathbf{x}_{i}^q(k) + \delta\mathbf{u}_{i}^q(k) \in \mathcal{M}$, the feasibility of the resulting discrete variable $\mathcal{X}^q_i(k)$ by random rounding, i.e.,  $\mathcal{X}^q_i(k) \in \mathcal{I}$, is not guaranteed. In fact, it is shown in \cite{zhang2019online} that there is no proper rounding scheme which can simultaneously preserve the feasibility $\mathcal{X}^q_i(k) \in \mathcal{I}$ and the unbiasedness of gradient estimation $\mathbb{E}[f^t(\mathcal{X}^q_i(k))]=$ $F(\mathbf{x}_{i}^q(k) + \delta\mathbf{u}_{i}^q(k))$. Hence, the so-called responsive bandit algorithm is developed in \cite{zhang2019online}, which requires bandit feedback in the form of function values queried at points that violate the constraint, i.e., $\mathcal{X}^q_i(k) \notin \mathcal{I}$. We deal with such a sampling violation issue in the next section.
\end{remark}

\section{Bounding the Sampling Violation}
In this section, we extend the previous algorithm development by incorporating a bounded stochastic pipage rounding (\texttt{B-SPR}) technique. By designing a sequence of progressively relaxed bounds, we show that the probability of occurrence of violating samples vanishes as $T\to \infty$, and furthermore, the cumulative sampling violation is sublinear in $T$.

\subsection{Bounded Stochastic Pipage Rounding}
The \texttt{B-SPR} scheme, as outlined in Algorithm \ref{algo:pipage-rounding}, is~a~generalized variant of the stochastic pipage rounding method. Precisely, the standard pipage rounding scheme converts any $M$-dimensional continuous vector $\mathbf{x} \in [0,1]^M$ to~a discrete one $\mathbf{z} \in \{0,1\}^M$ (given that $\mathbf{1}_M^\top \mathbf{x}$ is an integer). In contrast, our \texttt{B-SPR} scheme, denoted as~$\mathcal{R}_{b_l}^{b_u}(\cdot)$, generalizes the upper and lower bounds to any fractional numbers $b_l$ and $b_u$ such that $0 \le b_l < b_u \le 1$. We show in the following theorem that the promising properties of the stochastic pipage rounding are preserved by the \texttt{B-SPR} procedure.

\begin{algorithm}
\textbf{Input:} Lower and upper bounds $0 \le b_l<b_u \le 1$; input vector $\mathbf{x} \in [b_l,b_u]^M$.
\caption{-- Bounded Stochastic Pipage Rounding }\label{algo:pipage-rounding}
\begin{algorithmic}[1]
\State Initialize $\mathbf{y}^0 = \mathbf{x}$ and $\tau = 0$.
\While{there are two entries of $\mathbf{y}^\tau$ in $(b_l, b_u)$}
\State Randomly select $y^\tau(m), y^\tau(n) \in (b_l, b_u)$;
\State Denote $\Delta^\tau(m) = \min\big\{y^\tau(m)-b_l, b_u - y^\tau(n)\big\}$ and $\Delta^\tau(n) = \min\big\{y^\tau(n)-b_l, b_u - y^\tau(m)\big\}$;
\State Perform the randomized swapping:
\vspace{-5pt}
\begin{align*}
&\hspace{-10pt}\left\{\begin{matrix}
         y^{\tau + 1}(m) =  y^{\tau}(m) - \Delta^\tau(m), \\
         y^{\tau + 1}(n) =  y^{\tau}(n) + \Delta^\tau(m).
        \end{matrix}\right. \; \text{ w. p.} \;\;\frac{\Delta^\tau(n)}{\Delta^\tau(m)+\Delta^\tau(n)};\\
&\hspace{-10pt}\left\{\begin{matrix}
         y^{\tau + 1}(m) =  y^{\tau}(m) + \Delta^\tau(n), \\
         y^{\tau + 1}(n) =  y^{\tau}(n) - \Delta^\tau(n).
        \end{matrix}\right. \; \text{ w. p.} \;\;\frac{\Delta^\tau(m)}{\Delta^\tau(m)+\Delta^\tau(n)};
\end{align*}
\vspace{-10pt}
\State Update $y^{\tau + 1} (s)= y^{\tau} (s),\;\forall s \neq m, n$;
\State Let $\tau \leftarrow \tau +1$ and continue.
\EndWhile
\end{algorithmic}
\textbf{Output: $\mathbf{y}^\tau$}.
\end{algorithm}

\begin{theorem}\label{thm:pipage-rounding}
    For any integer $M \in \mathbb{Z}_{\ge 1}$ and any input vector $\mathbf{x} \in [b_l,b_u]^M$, the \texttt{B-SPR} procedure $\mathcal{R}_{b_l}^{b_u}(\cdot)$ in Algorithm \ref{algo:pipage-rounding} has to terminate in finite time $\mathcal{T} \le M-1$ and the following statements hold for \mbox{$\mathbf{z} = \mathcal{R}_{b_l}^{b_u}(\mathbf{x}) \in [b_l,b_u]^M$}: i) $\mathbf{1}_M^\top \mathbf{z} = \mathbf{1}_M^\top \mathbf{x} $; ii) $\sum_{m=1}^M \mathbbm{1}\{z(m) \in (b_l,b_u)\} \le 1$; and iii) $\mathbb{E} [ F(\mathbf{z})] \ge F(\mathbf{x})$ where the expectation is taken with respect to the stochasticity in $\mathcal{R}_{b_l}^{b_u}(\cdot)$ and $F(\cdot)$ is the multi-linear extension of a monotone submodular function.
\end{theorem}
\begin{proof}
    See Appendix III-A.
\end{proof}

Theorem \ref{thm:pipage-rounding} establishes a few fundamental properties of the \texttt{B-SPR} procedure. First, statement i) verifies that the procedure terminates in finite time while preserving the sum of the input vector. Second, statement ii) ensures that at most one element of the output vector $\mathbf{z}$ takes a value other than $b_l$ and $b_u$. More importantly, statement iii) shows that, for any multi-linear extension function, the expected function value does not decrease after applying the \texttt{B-SPR} procedure. This property is crucial for the development of algorithms in the sequel. Intuitively, \texttt{B-SPR} mirrors standard stochastic pipage rounding: it repeatedly transfers mass between two fractional coordinates while preserving their sum. The difference is that each coordinate is rounded toward the box $[b_l,b_u]$ rather than $\{0,1\}$. The swap probabilities are calibrated so that the multi-linear value does not decrease in expectation. This box-valued generalization precisely enables the progressive-bounding mechanism~developed~next.

\subsection{Distributed Online Algorithm with Progressively Bounded Stochastic Pipage Rounding}\label{subsec:PB-SPR}
Equipped with the above \texttt{B-SPR} technique, we are now able to tackle the issue of sampling violations as noted in Remark \ref{remark:violating-sample}. The enhanced distributed online algorithm is outlined in Algorithm \ref{algo:vs}, in which the distributed Frank-Wolfe phase is integrated with a progressively bounded stochastic pipage rounding (\texttt{PB-SPR}) scheme while the online learning phase follows exactly the same steps as in Algorithm \ref{algo:bandit-feedback}. The key idea here is to drive the local variable $[\mathbf{x}_i^q (k)]_i$ to be as close as possible to the vertex of $\widetilde{\mathcal{M}}_i$. To do so, we let the upper bound $b_u$ of the \texttt{PB-SPR} procedure increase progressively as $b_u = \delta+\alpha k/K$, while the lower bound remains unchanged as $b_l = \delta$ (see Line 6). Due to the features of the \texttt{B-SPR} scheme in Theorem \ref{thm:pipage-rounding}, it holds that $\mathbf{1}_{N_i}^\top[\mathbf{x}_i^q(k+1)]_i \le \delta N_i + \alpha \kappa_i k/K$. That is, there are at least $N_i - \kappa_i$ elements of the vector $[\mathbf{x}_i^q (k+1)]_i$ whose values are $\delta$. Given that $\delta$ vanishes as the time horizon $T \to \infty$, it is expected that the probability of generating a violating sample also vanishes accordingly.

Before proceeding to the probability analysis, we should highlight two challenges that need to be addressed. First, due to the \texttt{PB-SPR} procedure, the elementwise monotonicity of the local variables, i.e., $[\mathbf{x}_i^q (k+1)]_i \succeq [\mathbf{x}_i^q (k)]_i$, no longer holds. As a result, the maximum consensus scheme employed in Algorithms \ref{algo:full-info-feedback} and \ref{algo:bandit-feedback} is not applicable in this case. Nevertheless, thanks to the sum-preserving property of \texttt{B-SPR} (see Theorem \ref{thm:pipage-rounding} statement i)), it holds that $\mathbf{1}_{N_i}^\top[\mathbf{x}_i^q (k+1)]_i \ge \mathbf{1}_{N_i}^\top[\mathbf{x}_i^q (k)]_i$. As such, we adopt~a sum-maximum consensus (see Line 7 in Algorithm \ref{algo:vs}), whereby each local variable $[\mathbf{x}_i^q(k+1)]_{i'}$ is updated by selecting, from its neighbors, the vector with the largest aggregate value. 

More importantly, for the same reason, the discrepancy $\|\bar{\mathbf{x}}^q(k)-\mathbf{x}_i^q(k)\|$ can no longer be bounded as directly as in the previous analysis. Fortunately, we are able to show that the difference between $\mathbf{x}_i^q(k+1)$ and $\mathbf{x}_i^q(k)$ is bounded in expectation by the term $\mathcal{O}(1/K)$. We next establish such a result in Lemma \ref{lemma:rounding-norm-bound}, and then use it to show the desired consensus property in Lemma \ref{lemma:consensus-vs}.

\begin{algorithm}
\textbf{Input:} Horizon length $T$; block size $L$; number of blocks $Q = T/L$; number of gradient estimation steps per block \mbox{$K \le L$}; step-sizes $\rho_k$; communication graph $\mathcal{G}$; constraint $\mathcal{M}_i$; and parameters $\delta, \alpha$.
\caption{-- Distributed online algorithm (bounded violating samples)}\label{algo:vs}
\begin{algorithmic}[1]
\State Each agent $i\in[I]$ initializes {experts} $\mathcal{E}_i(k), \, k \in [K]$  for maximizing linear cost over the local constraint $\widetilde{\mathcal{M}}_i:= \alpha \mathcal{M}_i + \delta \cdot \mathbf{1}_{N_i}$.
\For{$q=1,2,\cdots,Q$}

\hspace{-20pt}\# \texttt{Distributed Frank-Wolfe Phase}
\State Initialize $\mathbf{x}_i^q (1) = \delta \cdot\mathbf{1}_{N}$ and $\mathbf{d}_i^q(0) = \mathbf{0}$.
\For{$k = 1,2,\cdots,K$}
\State Let $\mathbf{v}_i^q(k) \in \widetilde{\mathcal{M}}_i$ be the output of $\mathcal{E}_i(k)$;
\State Update 
% \vspace{-10pt}
$$\hspace{-10pt}\big[{\mathbf{x}}_i^q(k+1)\big]_i = \mathcal{R}_\delta^{\delta+\alpha k/K} \Big(\big[\mathbf{x}_i^q (k)\big]_i + \big(\mathbf{v}_i^q(k) - \delta\cdot\mathbf{1}_{N_i}\big)/K\Big);$$
\State Receive $\big[{\mathbf{x}}_j^q(k+1)\big]_{i'}$ from neighbors $j \in \mathcal{N}_i$ and update
$$\big[{\mathbf{x}}_i^q(k+1)\big]_{i'} = \big[{\mathbf{x}}_{j^\star}^q(k+1)\big]_{i'}, \quad \forall i' \neq i,$$ where $j^\star \in \arg\max_{j \in \mathcal{N}_i \cup \{i\}}\mathbf{1}_{N_{i'}}^\top\big[{\mathbf{x}}_j^q(k+1)\big]_{i'}.$
\EndFor

\hspace{-20pt}\# \texttt{Online Learning Phase}
\State Execute the steps as Lines 9--17 in Algorithm \ref{algo:bandit-feedback}.
\EndFor
\end{algorithmic}
\end{algorithm}

\begin{lemma}\label{lemma:rounding-norm-bound}
    For any integer $M \in \mathbb{Z}_{\ge 1}$ and $k \in [K]$, let  $\widetilde{\mathbf{z}}(k+1) = \mathbf{z}(k) + \mathbf{v}(k)/K$ where $\mathbf{v}(k) \in [0,\alpha]^M$ and $\mathbf{z}(k) \in \mathbb{R}^M$ be the output of $\mathcal{R}_{\delta}^{\delta+\alpha (k-1)/K }(\cdot)$, then it holds 
  \begin{align}
    \mathbb{E}\Big[\|\mathbf{z}(k+1) - \mathbf{z}(k)\|_1\Big] \le 2 \alpha M^2/K,
  \end{align}
  where $\mathbf{z}(k+1) = \mathcal{R}_{\delta}^{\delta+\alpha k/K }\big(\widetilde{\mathbf{z}}(k+1)\big)$ and the expectation is taken with respect to the randomness in $\mathcal{R}_{\delta}^{\delta+\alpha k/K }(\cdot)$.
\end{lemma}
\begin{proof}
    See Appendix III-B.
\end{proof}

\begin{lemma}\label{lemma:consensus-vs}
     Let $\bar{\mathbf{x}}^q(k):= \big[[\mathbf{x}_1^q(k)]_1^\top,  \cdots\!, [\mathbf{x}_I^q(k)]_I^\top\big]^\top$~where $\mathbf{x}_i^q(k)$ is generated by Algorithm \ref{algo:vs}, the following statements hold for $\forall q \in [Q]$ and $k \in [K+1]$:
 \begin{enumerate}[i)]
     \item $\bar{\mathbf{x}}^q(k) \in \widetilde{\mathcal{M}}$ and $\mathbf{x}_i^q(k) \in \widetilde{\mathcal{M}}$, $\forall i\in[I]$;
     \item $\mathbb{E}\big[\|\bar{\mathbf{x}}^q(k) - \mathbf{x}_i^q(k)\|\big] \hspace{-1pt}\le\hspace{-1pt} 2\alpha N^2 d(\mathcal{G})/K$, $\forall i \in[I]$.
 \end{enumerate}
\end{lemma}
\begin{proof}
    See Appendix III-C.
\end{proof}

With the help of the consensus result in Lemma~\ref{lemma:consensus-vs}, we~are now ready to show the cumulative regret bound of Algorithm \ref{algo:vs} in the following theorem.

\begin{theorem}\label{thm:regret-vs}
    Under the same conditions in Theorem \ref{thm:regret-bandit}, the expected cumulative regret of Algorithm~\ref{algo:vs} satisfies
    \begin{equation}\label{ineq:regret-bound-vs}
     \begin{aligned}
        \mathbb{E}[\mathcal{R}_T] &\le\Big((1-1/e)L_1C_\mathcal{M}C_\delta + (2-1/e)L_1 C_\delta  + 3/4\cdot C_2/C_\delta + C_0 I + 3/4\cdot {D}^2/C_\delta\Big)\cdot T^{8/9}\\
          & \quad +3/4\cdot C_3C_\delta\cdot  T^{2/3} + 2L_1 N^3d(\mathcal{G}) \sum_{i=1}^I {D}_i/C_\delta \cdot T^{4/9} +L_2{D}^2 /2 \cdot T^{1/3}.
    \end{aligned}   
    \end{equation}
\end{theorem}
\begin{proof}
    See Appendix III-D.
\end{proof}

It has been shown by Theorem \ref{thm:regret-vs} that Algorithm \ref{algo:vs} retains the same $\mathcal{O}(T^{8/9})$ regret rate as Algorithm \ref{algo:bandit-feedback} in Theorem \ref{thm:regret-bandit}. Compared to \eqref{ineq:regret-bound-bandit}, the only difference is the second last term of \eqref{ineq:regret-bound-vs}, which reflects the consensus result of \texttt{PB-SPR} procedure (see Lemma \ref{lemma:consensus-vs} statement ii)). Since this term grows as $\mathcal{O}(T^{4/9})$, it is non-dominant and the incorporation of \texttt{PB-SPR} does not compromise the regret bound. We next show that, thanks to \texttt{PB-SPR}, the probability of sampling violation vanishes as $T \to \infty$.

\begin{theorem}\label{thm:vs-complexity}
    Under the same conditions in Theorem \ref{thm:regret-bandit}, the probability of a violating sample for each agent $i\in[I]$ at each step of Algorithm \ref{algo:vs} is bounded by 
    \begin{align}\label{ineq:vs-prob-bound}
        \mathbb{P}\Big(\mathcal{X}^q_i(k) \notin \mathcal{I}\Big) \le 2(N {-} \kappa)C_\delta \cdot T^{-1/9},\;\forall k\in[K], q\in[Q].
    \end{align}
    Furthermore, the cumulative sampling violation satisfies
    \begin{align}\label{ineq:vs-regret-bound}
        \sum_{q=1}^Q \sum_{k=1}^K  \mathbb{E} \Big[\mathbbm{1}\big\{\mathcal{X}^q_i(k) \notin \mathcal{I}\big\}\Big] \le 2(N - \kappa)C_\delta \cdot T^{7/9}.
    \end{align}
\end{theorem}
\begin{proof}
    See Appendix III-E.
\end{proof}

Theorem \ref{thm:vs-complexity} shows that, for each agent $i \in [I]$, the probability of generating a violating sample at each step vanishes at the rate $\mathcal{O}(T^{-1/9})$, so the cumulative sampling violation grows sublinearly as $\mathcal{O}(T^{7/9})$. Since sampling occurs only at $K$ out of $L$ steps per block, with $K = T^{2/3}$ and $Q =T^{2/9}$ the cumulative violation would be naturally bounded by $\mathcal{O}(T^{8/9})$; the vanishing violation probability in \eqref{ineq:vs-prob-bound} improves this to $\mathcal{O}(T^{7/9})$ as in \eqref{ineq:vs-regret-bound}. Moreover, as shown next, $T^{7/9}$ also serves as a lower bound in certain scenarios, implying that the bound \eqref{ineq:vs-regret-bound} is nearly optimal and not improvable under certain conditions.
\begin{proposition}\label{prop:vs-lower-bound}
    Suppose that $\mathbf{1}_{N_i}^\top\hspace{-1pt} \mathbf{v}_i^q(k) \!=\! \delta N_i + \alpha\kappa_i$, then for any operator $\mathcal{P}(\cdot): [\delta,\alpha +\delta]^{N_i} \to [\delta,\alpha +\delta]^{N_i}$ with the sum-preserving property such that 
    \begin{align}
        \big[{\mathbf{x}}_i^q(k+1)\big]_i = \mathcal{P} \Big(\big[\mathbf{x}_i^q (k)\big]_i + \big(\mathbf{v}_i^q(k) - \delta\cdot\mathbf{1}_{N_i}\big)/K\Big),
    \end{align}
    the cumulative sampling violation satisfies
    \begin{align}
        \sum_{q=1}^Q \sum_{k=1}^K \mathbb{E} \Big[\mathbbm{1}\big\{\mathcal{X}^q_i(k) \notin \mathcal{I}\big\}\Big] \ge\Omega(T^{7/9}).
    \end{align}
\end{proposition}
\begin{proof}
    See Appendix III-F.
\end{proof}

It has been shown in Proposition \ref{prop:vs-lower-bound} that, under the condition $\mathbf{1}_{N_i}^\top\hspace{-1pt} \mathbf{v}_i^q(k) =  \delta N_i + \alpha\kappa_i$ and for any sum-preserving operator $\mathcal{P}(\cdot)$ used to replace \texttt{PB-SPR} in Algorithm \ref{algo:vs}, the cumulative sampling violation is bounded by $\Omega(T^{7/9})$. It is worth noting that $\mathbf{1}_{N_i}^\top\hspace{-1pt} \mathbf{v}_i^q(k) =  \delta N_i + \alpha\kappa_i$ often holds due to the nature of Frank-Wolfe. To elaborate on this, recall in \eqref{frank-wolfe} that the vector $\mathbf{v}_i^q(k)$ is obtained by maximizing a linear function under the polytope constraint. Hence, $\mathbf{v}_i^q(k)$ typically corresponds to an extreme point of $\widetilde{\mathcal{M}}_i$, which naturally satisfies the condition as considered in Proposition \ref{prop:vs-lower-bound}. In fact, when considering the full-information feedback model as discussed in Sec. \ref{sec:full-info}, the gradient $\nabla F^t(\cdot)$ is componentwise nonnegative, making the resulting vector $\mathbf{v}_i^q(k)$ readily satisfy the above condition.

\section{Numerical Example}

In this section, we evaluate the performance of the proposed distributed online algorithms through numerical examples. We consider the multi-agent information harvesting problem studied in \cite{rezazadeh2023distributed}. Specifically, let $\mathcal{D} \subset \mathbb{R}^2$ denote a set of information sources and $\mathcal{B} \subset \mathbb{R}^2$ denote a set of information retrieval points. A team of $I$ agents is deployed, where each agent $i$ is allowed to install at most $\kappa_i \in \mathbb{Z}_{>0}$ information harvesting devices at locations selected from its candidate subset $\mathcal{B}_i \subset \mathcal{B}$. The objective of agents is to determine the device deployment~so as to minimize the aggregate distance between all information sources and devices, 
which is equivalent to maximizing a normalized monotone submodular function \cite{rezazadeh2023distributed}.

 \begin{figure}[H]
  \centering
  \subfloat[$4$-agent network]{\includegraphics[width=0.35\linewidth]{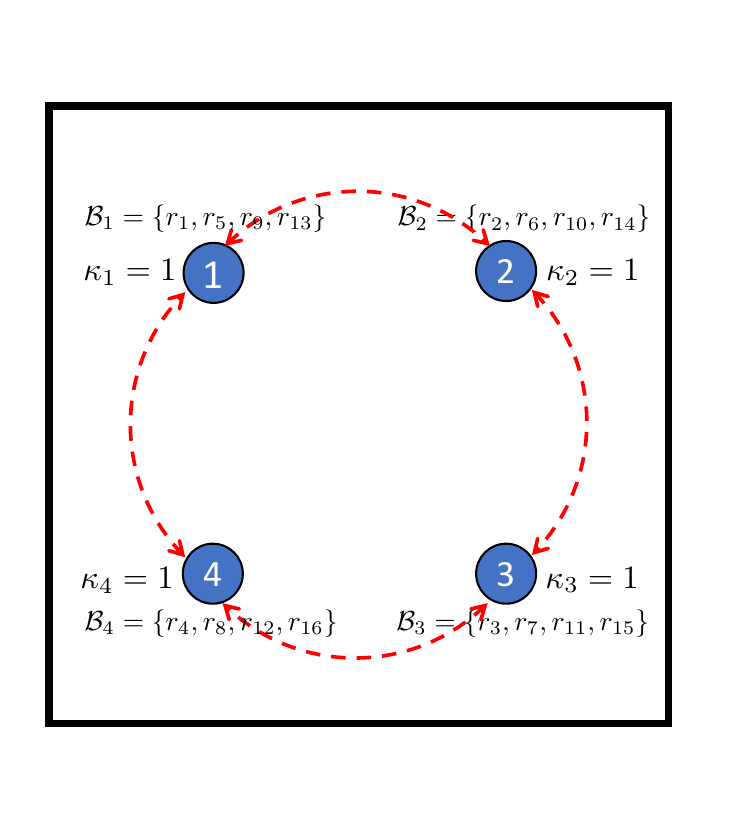} \label{subfig:1-a}}
  \subfloat[Info \& retrieval points]{\includegraphics[width=0.35\linewidth]{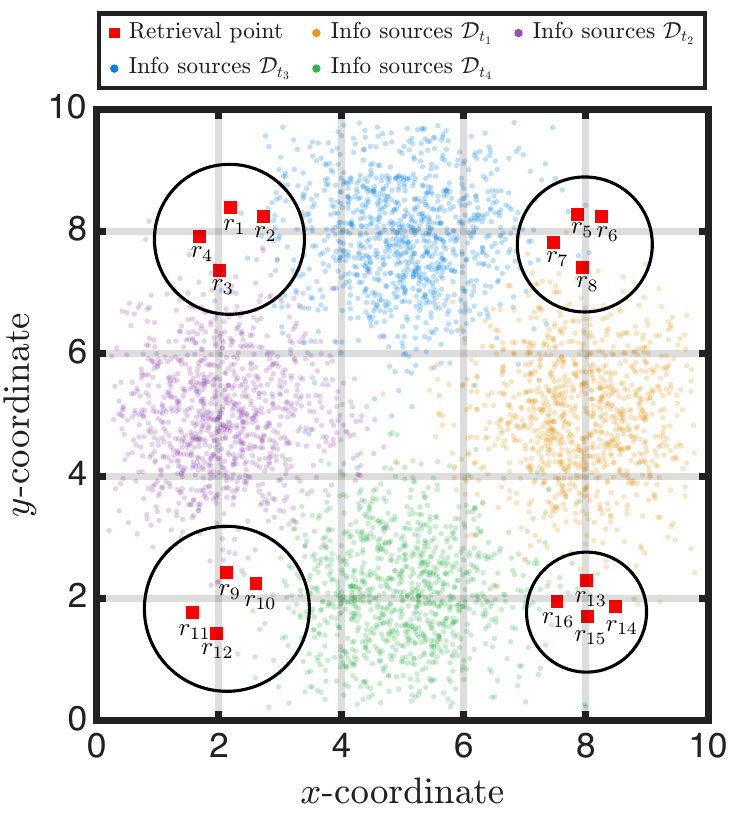} \label{subfig:1-b}}
  \caption{Simulation setup: $4$ agents connected via a ring graph, where each agent selects at most one retrieval point from its candidate set $\mathcal{B}_i$ for device deployment. 
  The retrieval points are partitioned into four clusters. 
  To maximize the overall objective value, agents are encouraged to cooperatively avoid overlap, i.e., to avoid selecting retrieval points from the same cluster.
  }
  \label{fig:sim-setup}
\end{figure}

In this simulation, we consider $|\mathcal{D}| = 1000$ information sources and $|\mathcal{B}| = 16$ retrieval points. The positions of information sources and retrieval points are spread over a $10 \times 10$ region, as shown in Fig.~\ref{fig:sim-setup}\subref{subfig:1-b}. To conform with the online problem setting, we allow the configuration of information sources $\mathcal{D}_t$ to vary randomly over time $t \in [T]$. More specifically, at each time $t$, the set $\mathcal{D}_t$ is generated according to a Gaussian kernel whose mean moves continuously within the region over the time horizon. The knowledge of $\mathcal{D}_t$ is not revealed to the agents prior to making decisions at the time $t$. A team of $I =4$ agents is employed, whose communication network and local candidate set $\mathcal{B}_i$ are shown in Fig.~\ref{fig:sim-setup}\subref{subfig:1-a}. Note that all online algorithms implemented in this section apply RFTL \cite{hazan2016introduction} as the scheme of tracking the~expert. In addition, the cumulative regret $\widetilde{\mathcal{R}}_T$ is computed against the best decision in hindsight\footnote{The results obtained in this example consistently outperform the $(1-1/e)$-approximate solution. Hence, for a clearer empirical comparison, we compute regret w.r.t. the optimal solution, rather than the $(1-1/e)$-approximate one.}, by a brute-force search.

\begin{figure*}
    \centering
    \subfloat[regret performance against baselines]{\includegraphics[width=0.33\linewidth]{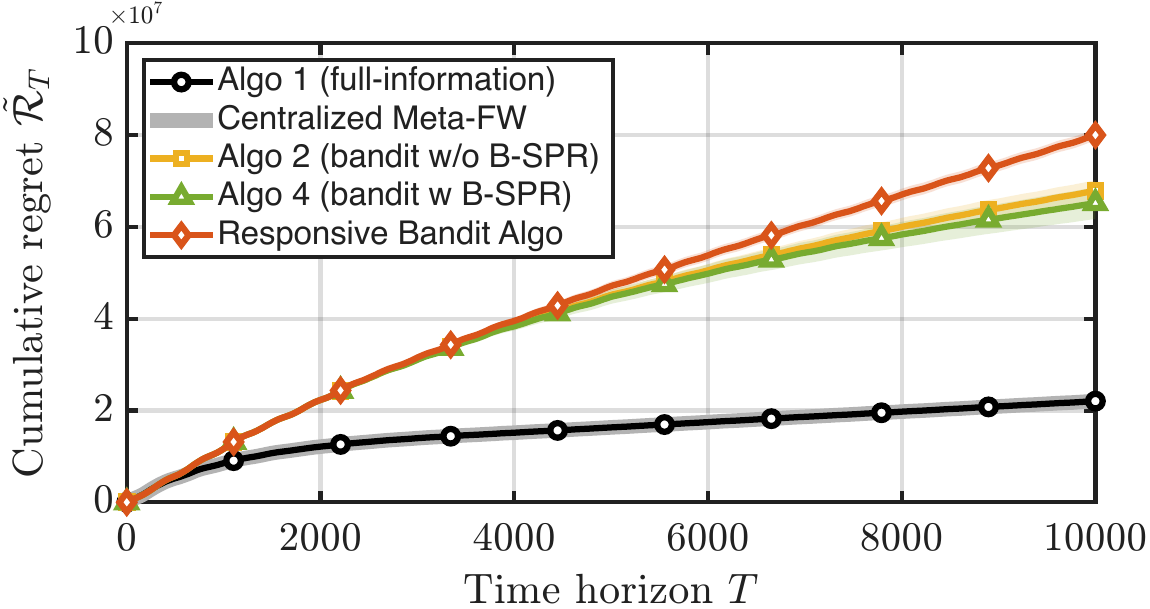}\label{subfig:regret}}
  \subfloat[\# of violating samples per round]{\includegraphics[width=.32\linewidth]{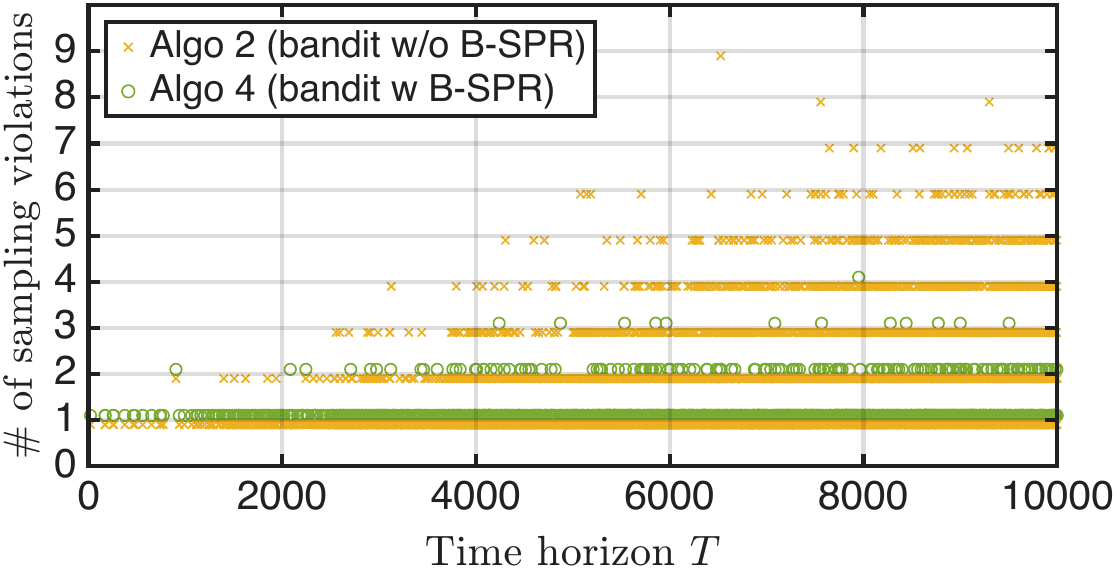} \label{subfig:vs-a}}
  \subfloat[\# of violating samples per block]{\includegraphics[width=0.335\linewidth]{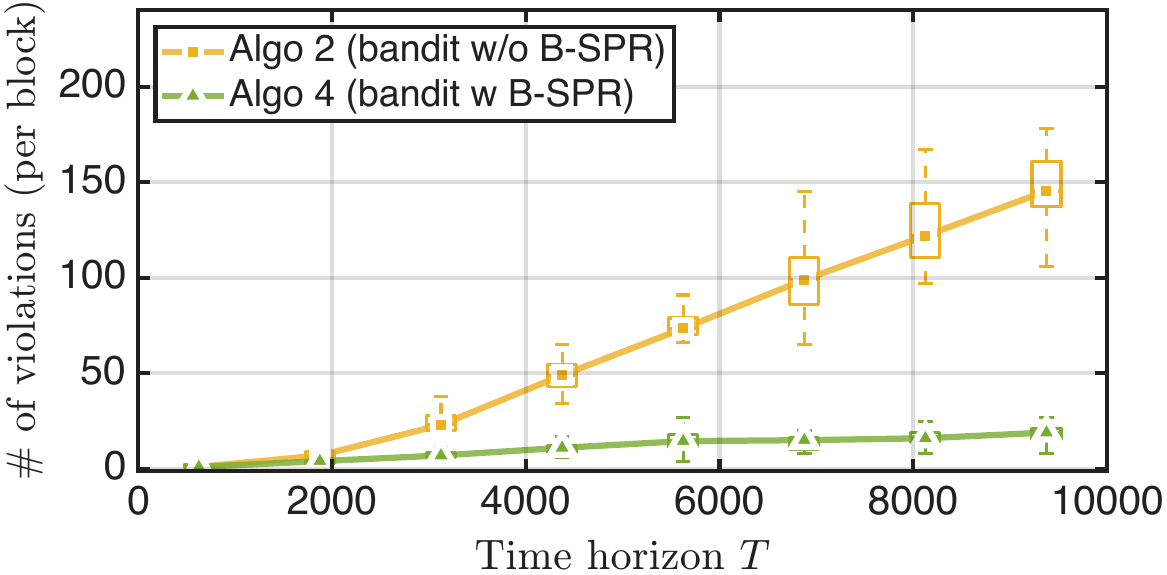} \label{subfig:vs-b}}
  \caption{Performance of the proposed algorithms against centralized baselines: (a) cumulative regrets of the proposed algorithms (all curves are averaged over $20$ independent simulation trials); (b) number of violating samples per round $t$ (aggregated over $20$ independent trials); (c) cumulative number of violating samples per block $q$.}
  \label{fig:performance}
\end{figure*}

We first examine the performance of Algorithms \ref{algo:full-info-feedback}, \ref{algo:bandit-feedback}~and \ref{algo:vs}, whose cumulative regrets and~sampling violation behaviors are shown in Fig. \ref{fig:performance}. By {sampling violation}, we mean a gradient-estimation query whose~sample~is~infeasible. We note that such infeasible queries are~an unavoidable by-product of the one-point~gradient~estimator~\cite{zhang2019online}, and the role of our \texttt{PB-SPR} scheme is precisely to drive their frequency toward zero. In addition, we~also compare the performance of our~algorithms against two centralized baselines: i) the classical Meta-Frank-Wolfe, i.e., the centralized counterpart of Algorithm \ref{algo:full-info-feedback}; and ii) the responsive bandit algorithm in~\cite{zhang2019online}, which operates in a bandit feedback setting where an infeasible query is permitted but earns zero reward. It can be observed from Fig.~\ref{fig:performance}\subref{subfig:regret} that the cumulative regrets of our algorithms all grow sublinearly with respect to $T$. As expected, the full-information Algorithm \ref{algo:full-info-feedback} attains the lowest regret, while the bandit Algorithms~\ref{algo:bandit-feedback} and \ref{algo:vs} incur a moderately larger regret owing to the more restrictive feedback. Importantly, the incorporation of \texttt{PB-SPR} into the bandit algorithm does not compromise the regret performance. In fact, as suggested in Theorem~\ref{thm:pipage-rounding}, the function value~is non-decreasing in expectation after the \texttt{B-SPR} operation, which is reflected by a slight improvement of~Algorithm~\ref{algo:vs} over Algorithm \ref{algo:bandit-feedback}. In addition, it is also observed from Fig.~\ref{fig:performance}\subref{subfig:regret} that~Algorithm~\ref{algo:full-info-feedback} and the centralized Meta-FW achieve essentially \mbox{coincident}~regret,~which~implies~that operating over the distributed~network incurs negligible loss. Moreover, since the responsive bandit algorithm frequently
queries and thus is charged for infeasible~samples, it incurs a higher regret than our bandit Algorithms~\ref{algo:bandit-feedback} and~\ref{algo:vs}. Most importantly, thanks to \texttt{PB-SPR}, the occurrence of sampling violations is substantially reduced, by roughly $84\%$ on this simulation instance, as shown in Fig.~\ref{fig:performance}\subref{subfig:vs-a}. As the experts progressively concentrate the iterates towards a vertex, the per-block number of violating samples grows for the standard bandit algorithm, whereas \texttt{PB-SPR} keeps it consistently and increasingly lower, as shown in Fig.~\ref{fig:performance}\subref{subfig:vs-b}.

To further study how the communication graph shapes the behavior of our distributed algorithms, we next vary the number of agents by $I\in\{4,8,12,16,24,32,40\}$ and consider three different network topologies: i) an Erd\H{o}s--R\'enyi \emph{random graph}, which remains well-connected~with small diameter; ii) a two-dimensional \emph{grid graph} with medium diameter; and iii) a \emph{ring graph} with the largest diameter, as shown in Fig.~\ref{fig:sim-setup}\subref{subfig:1-a}. We note that for each instance, the ground set scales as $N = 4I$. Hence, a brute-force search for the optimal solution may be intractable, and thus the regret here is measured against an offline greedy benchmark. Besides the regret performance and sampling-violation behavior, we also track the metric of \emph{consensus error}, which quantifies how closely the agents agree on the 
global decision. Precisely, the consensus error $\mathcal{J}$ is defined by
\begin{align}\label{eq:consensus-error}
  \mathcal{J} := \frac{1}{(I-1)N}\sum_{j=1}^{I}\sum_{i\neq j}
  \big\| [\mathbf{x}_{i}^q(K{+}1)]_j - [\mathbf{x}_{j}^q(K{+}1)]_j\big\|_2,
\end{align}
evaluated at all executed iterates. By the definition, we know that the consensus error
$\mathcal{J}=0$ if and only if all agents hold the same iterate.

The numerical results in Fig.~\ref{fig:topology} reveal three effects.~First, the
consensus error is topology-sensitive, and it increases as the diameter of the graph becomes large.  Second, the
regret is remarkably robust to the topology. Specifically, the per-round regret of the three
topologies differs by less than $1\%$ at every instance. This is mainly due to the fact that each agent's executed
reward is dominated by its own block, which it knows exactly, while the
topology only delays its estimate of the others' blocks. Third, the
sampling-violation rate is driven by problem scale rather than topology. It
rises with $I$, consistent with the
$(N-\kappa)$ dependence in Theorem \ref{thm:vs-complexity}, with only a mild
secondary topology effect.
The price of this robustness is that \texttt{PB-SPR} incurs a larger consensus
error than the standard bandit method, consistent with the $\mathcal{O}(N^2/K)$ consensus-error
term in the analysis.

\begin{figure}
  \centering
  \subfloat[consensus error]{\includegraphics[width=0.45\linewidth]{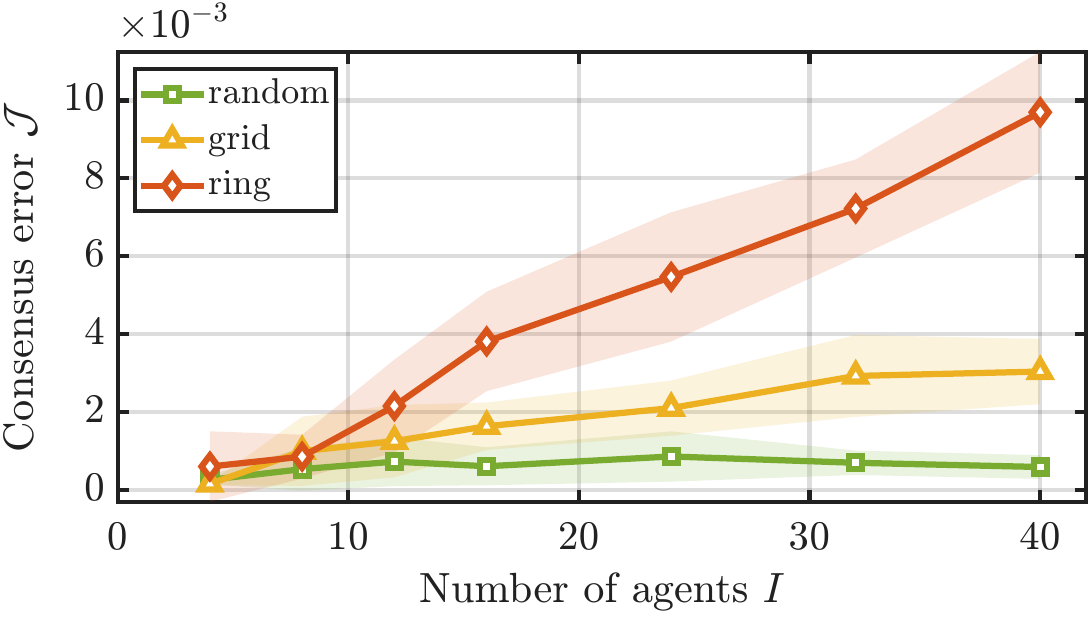}\label{subfig:topo-ce}}
  \subfloat[regret \& violation rate]{\includegraphics[width=0.52\linewidth]{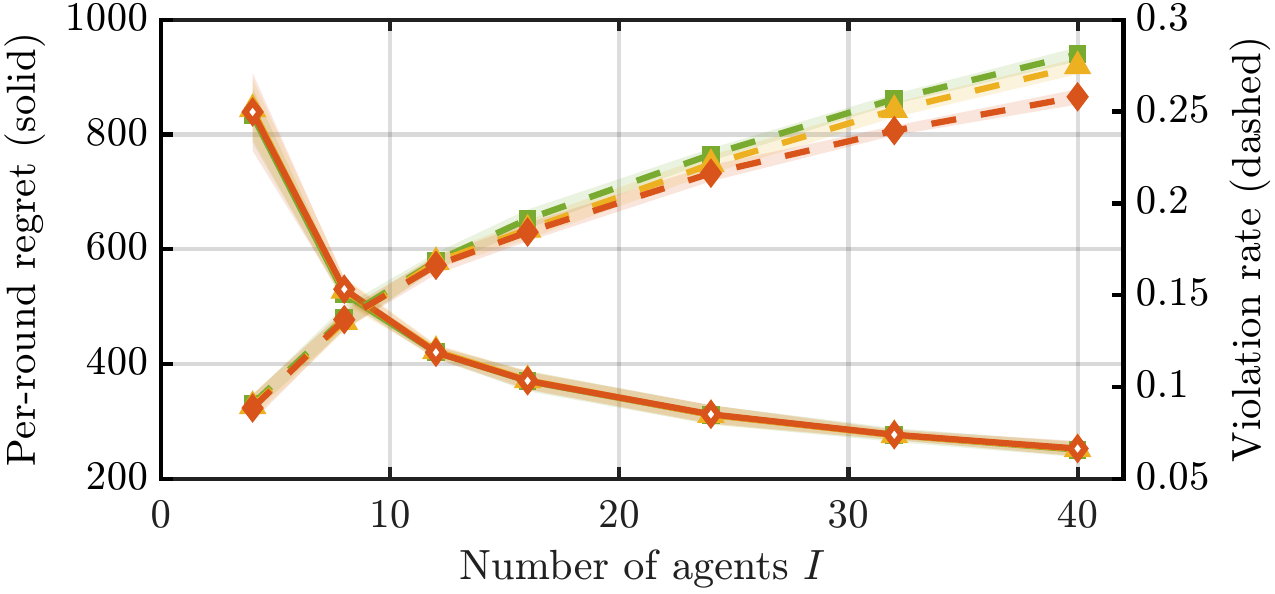}\label{subfig:topo-rv}}
  \caption{Effect of the communication topology on Algorithm~\ref{algo:vs} with
  \texttt{PB-SPR} (all curves are averaged over 10 independent simulation trials):
  (a) consensus error $\mathcal{J}$, as defined in \eqref{eq:consensus-error}; and
  (b) per-round regret (left axis, solid lines) and sampling violation rate (right
  axis, dashed lines).}
  \label{fig:topology}
\end{figure}

\section{Conclusion}
This paper studied the distributed submodular maximization problem in the online setting. An algorithmic framework was proposed which accommodates both full-information and bandit feedback models. Sublinear $(1-1/e)$-regret was established in both settings, matching the state-of-the-art centralized rates. To address the sampling violation issue which is inherent to bandit feedback with continuous relaxation, we further developed the \texttt{B-SPR} scheme and proved that the probability of sampling violation vanishes asymptotically. Numerical results validated the effectiveness of our algorithms. Future work will be devoted to developing distributed online algorithms that ensure hard per-round feasibility, as demanded by safety-critical applications. It would be of interest to extend the present framework beyond the partition matroid to more general constraints.

\bibliographystyle{unsrt}
\bibliography{ref}

\newpage
\onecolumn
\renewcommand{\thesubsection}{\Alph{subsection}}
\setcounter{subsection}{0}
\section*{Appendix}
{
\hypersetup{linkcolor=blue}
  \tableofcontents
}

\clearpage
\section*{Appendix I: Proofs Pertaining to Algorithm~\ref{algo:full-info-feedback}}
\addcontentsline{toc}{section}{Appendix I: Proofs Pertaining to Algorithm~\protect\ref{algo:full-info-feedback}}\label{append:A}
\subsection{Proof of Lemma~\ref{lemma:consensus-full-info}}
By the initialization $\mathbf{x}_i^q(1) = \mathbf{0}$ and $\mathbf{v}_i^q(k) \in \mathcal{M}_i$,~it is easy to see that $[\mathbf{x}_i^q(k)]_i \in \mathcal{M}_i$ for $\forall i \in [I]$, and thus $\bar{\mathbf{x}}^q(k) \in \mathcal{M}$~by the definition of $\bar{\mathbf{x}}^q(k)$. In addition, due to the updating rule~of $[\mathbf{x}_i^q(k)]_{i'}$ for $\forall i' \neq i$ (see Line~7~in Algorithm~\ref{algo:full-info-feedback}), we can have that $\mathbf{0} \preceq [\mathbf{x}_i^q(k)]_{i'} \preceq [\mathbf{x}_{i'}^q(k)]_{i'}$.~As a result,  $[\mathbf{x}_i^q(k)]_{i'} \in \mathcal{M}_{i'}$ holds for $\forall i' \neq i$ and thus $\mathbf{x}_i^q(k) \in \mathcal{M}$ for $\forall i \in [I]$. 

Notice that $\mathbf{1}_{N_i}^\top[\mathbf{x}_i^q(k)]_i \le \mathbf{1}_{N_i}^\top[\mathbf{x}_i^q(k-1)]_i + \kappa_i/K$. Hence, applying the inequality for consecutive $d(\mathcal{G})$ steps yields,
\begin{align}\label{ineq:proof-lemma-1-1}
    \mathbf{1}_{N_i}^\top[\mathbf{x}_i^q(k)]_i \le \mathbf{1}_{N_i}^\top[\mathbf{x}_i^q(k-d(\mathcal{G}))]_i + \kappa_id(\mathcal{G})/K.
\end{align}
Recall that $d(\mathcal{G})$ denotes the diameter of the graph $\mathcal{G}$, meaning that every agent $i' \neq i$ can be reached from the agent $i$ by at most $d(\mathcal{G})$-hop communications. Thus, it holds for $\forall i' \neq i$ that $\mathbf{1}_{N_i}^\top[\mathbf{x}_{i'}^q(k)]_{i} \ge \mathbf{1}_{N_i}^\top[\mathbf{x}_{i}^q(k -d(\mathcal{G}))]_{i}$. Combining this inequality together with~\eqref{ineq:proof-lemma-1-1} gives
\begin{align}\label{ineq:proof-lemma-1-2}
\mathbf{1}_{N_i}^\top[\mathbf{x}_{i'}^q(k)]_{i} \ge \mathbf{1}_{N_i}^\top[\mathbf{x}_{i}^q(k)]_{i}  - \kappa_id(\mathcal{G})/K.
\end{align}
Therefore, it follows that
\begin{equation}
\begin{aligned}
  \|\bar{\mathbf{x}}^q(k) - \mathbf{x}_i^q(k)\| \hspace{-1pt} \le \hspace{-1pt} \mathbf{1}_{N}^\top\big(\bar{\mathbf{x}}^q(k) - \mathbf{x}_i^q(k)\big) \hspace{-1pt}\le\hspace{-1pt} \kappa d(\mathcal{G})/K,
\end{aligned}
\end{equation}
where the first inequality is due to the facts that $\|\mathbf{x}\| \le \|\mathbf{x}\|_1$ and  $[\mathbf{x}_i^q(k)]_{i'} \preceq [\mathbf{x}_{i'}^q(k)]_{i'}$ holds for $\forall i' \neq i$, and the second one follows by summing \eqref{ineq:proof-lemma-1-2} for $\forall i \in [I]$ and $\kappa = \sum_{i \in [I]} \kappa_i$. The proof of Lemma~\ref{lemma:consensus-full-info} is completed.

\subsection{Proof of Theorem~\ref{thm:regret-full-info}}

To facilitate the subsequent proofs, let us first notice by the definitions in~\eqref{eq:def-M} and~\eqref{eq:def-Mi} that, the diameters of polytopes~$\mathcal{M}_i$ and $\mathcal{M}$ have $D_i \hspace{-1pt}=\hspace{-1pt} \max_{\mathbf{x}, \mathbf{y} \in \mathcal{M}_i} \|\mathbf{x} - \mathbf{y}\| \hspace{-1pt}=\hspace{-1pt}\sqrt{ \min\{2\kappa_i, N_i\}}$~and $D \hspace{-1pt}=\hspace{-1pt} \max_{\mathbf{x}, \mathbf{y} \in \mathcal{M}} \|\mathbf{x} - \mathbf{y}\| \hspace{-1pt}=\hspace{-1pt} \sqrt{\sum_{i=1}^I \min\{2\kappa_i, N_i\}}$,~respectively. In addition, recall that $f^t(\cdot)$ is~monotone and submodular with $\max_{\mathcal{S} \in \mathcal{I}} f^t (\mathcal{S}) \le f^\text{max}$, it can be shown (see Lemma~3 in~\cite{chen2020black}) that its multi-linear extension $F^t(\cdot)$~is $L_1$-Lipschtiz and $L_2$-smooth with  $L_1 = 2 \sqrt{N}f^\text{max}$ and $L_2 = 4Nf^\text{max}$.

Let $\mathcal{S}^\star \in \arg\max_{\mathcal{S} \in \mathcal{I}} \sum_{t=1}^T f^t(\mathcal{S})$ and $\widetilde{\mathbf{x}}^\star \in \mathcal{M}$ be the continuous variable so that $\sum_{t=1}^T F^t(\widetilde{\mathbf{x}}^\star)=\sum_{t=1}^T f^t(\mathcal{S}^\star)$. Likewise, let $\mathbf{x}^\star \in \arg\max_{\mathbf{x} \in \mathcal{M}} \sum_{t=1}^T F^t(\mathbf{x})$, it is clear that $\sum_{t=1}^T F^t({\mathbf{x}}^\star) \ge \sum_{t=1}^T F^t(\widetilde{\mathbf{x}}^\star)$. Considering that each discrete decision $\mathcal{S}^t$ is obtained by lossless rounding on $\bar{\mathbf{s}}^t= [\mathbf{s}_1(t)^\top,\cdots,\mathbf{s}_I(t)^\top]^\top$, i.e., $\mathbb{E}[f^t(\mathcal{S}^t)] \ge F^t(\bar{\mathbf{s}}^t)$,~it~has
\begin{equation}\label{eq:regret-1-continous-regret}
\begin{aligned}
\mathcal{R}^c_T &:= (1-1/e)\cdot \sum_{t=1}^TF^t(\mathbf{x}^\star) - \sum_{t=1}^T F^t( \bar{\mathbf{s}}^t) \\
&\ge\mathbb{E}\Big[(1-1/e)\cdot \sum_{t=1}^Tf^t(\mathcal{S}^\star) - \sum_{t=1}^T f^t( \mathcal{S}^t)\Big].
\end{aligned}
\end{equation}
Therefore, to prove the regret bound in Theorem~\ref{thm:regret-full-info}, it suffices to next show that $\mathcal{R}_T^c$ is no larger~than the same bound.

Recall the definition of $\bar{\mathbf{x}}^q(k)$ in Lemma~\ref{lemma:consensus-full-info} and let $\bar{\mathbf{v}}^q(k) := [\mathbf{v}_1^q(k)^\top,\cdots,\mathbf{v}_I^q(k)^\top]^\top \in \mathcal{M} $, it is straightforward to verify that $\bar{\mathbf{x}}^q(k+1) = \bar{\mathbf{x}}^q(k)  + \bar{\mathbf{v}}^q(k) /K, \forall k \in [K]$. Now, define a new vector $\bar{\mathbf{d}}_i^q(k) \in \mathbb{R}^N$ for every $i\in[I]$ which is generated by the following~recursion with $\bar{\mathbf{d}}_i^q(0) = \mathbf{0}$,
  \begin{align}\label{eq:d-bar}
    \bar{\mathbf{d}}_i^q(k) = (1-\rho_k)\bar{\mathbf{d}}_i^q(k-1) + \rho_k \nabla F^{t_i^{q,k}}\big(\bar{\mathbf{x}}^q(k)\big).
  \end{align}
Furthermore, let $\widetilde{\mathbf{d}}^q(k) := 1/I\cdot\sum_{i=1}^I\bar{\mathbf{d}}_i^q(k)$. Then, by the facts that $\{t_i^{q,1}, t_i^{q,2}, \cdots, t_i^{q,K}\}$ is a uniformly random permutation of $\{1,2,\cdots, K\}$ and is independent for $\forall i\in[I]$, we can have
  \begin{align}\label{eq:independent-gradient}
     \mathbb{E}\Big[\widetilde{\mathbf{d}}^q(k) - \bar{\mathbf{d}}_i^q(k) \Big] = 0.
  \end{align}
Next, by the same steps in the proof of Lemma~4 in~\cite{zhang2019online}, we can obtain the following result.
\begin{lemma}\label{lemma:regret-1-key-step-1}
    Define $\bar{F}_i^{q,k}(\cdot): = \sum_{l=k+1}^K F^{t_i^{q,l}}(\cdot)/(K-k)$,~it holds for $\forall q \in [Q]$ and $i \in [I]$ that
    \begin{equation}
    \begin{aligned}
      &\mathbb{E}\Big[(1-1/e) \cdot\bar{F}_i^{q,0}(\mathbf{x}^\star) - \bar{F}_i^{q,0}\big(\bar{\mathbf{x}}^q(K+1)\big)\Big] \\
      &\le  1/K \cdot\sum_{k=1}^K\frac{1}{2\beta(k)}\mathbb{E} \Big[\|\nabla \bar{F}_i^{q,k-1}(\bar{\mathbf{x}}^q(k)) - \bar{\mathbf{d}}_i^q(k)\|^2 \Big]\\
      & \hspace{10pt}+  1/K \cdot \sum_{k=1}^K(1-1/K)^{K-k} \cdot\mathbb{E}\big[\langle \bar{\mathbf{d}}_i^q(k), \mathbf{x}^\star - \bar{\mathbf{v}}^q(k) \rangle\big] \\
      & \hspace{10pt} + 1/K\cdot \sum_{k=1}^K \frac{\beta(k)D^2}{2}+ \frac{L_2 D^2}{2K},
    \end{aligned}
    \end{equation}
    where $\{\beta(k)\}_{k\in [K]}$ is an arbitrary~positive sequence.
\end{lemma}

Now, with the aid of the above Lemma~\ref{lemma:regret-1-key-step-1}, we can show an intermediate result for the upper bound of $\mathbb{E}[\mathcal{R}_T^c]$.
\begin{lemma}\label{lemma:regret-1-intermediate-step}
    Let $\mathcal{R}_T^c$ be defined in~\eqref{eq:regret-1-continous-regret}, then it holds that
    \begin{equation}\label{ineq:regret-1-intermediate-step}
        \begin{aligned}
                \mathbb{E}[\mathcal{R}^c_T] \le &\sum_{q=1}^Q \sum_{i=1}^I\sum_{k=1}^K \frac{1}{2\beta(k)I} \mathbb{E}\Big[{\underbrace{\|\nabla \bar{F}_i^{q,k-1}(\bar{\mathbf{x}}^q(k)) - \bar{\mathbf{d}}_i^q(k)\|^2}_{:= \Delta_i^1(k) -- \text{stochastic gradient error}}}\Big] \\ 
    &+ \sum_{q=1}^Q \sum_{i=1}^I\sum_{k=1}^K (1-1/K)^{K-k}  D_i\cdot\mathbb{E}\Big[\underbrace{\|\bar{\mathbf{d}}_i^q(k) - {\mathbf{d}}_i^q(k)\|}_{:=\Delta_i^2(k) -- \text{consensus error}}\Big]\\
    &+ K\sum_{i=1}^I\mathcal{R}_Q^{\mathcal{E}_i}+ \frac{QD^2}{2}\sum_{k=1}^K \beta(k) + \frac{L_2QD^2}{2},
        \end{aligned}
    \end{equation}
      where $\mathcal{R}_Q^{\mathcal{E}_i}$ represents the cumulative regret of expert $\mathcal{E}_i(k)$.
\end{lemma}
\begin{proof}
    Let us first note that, by the definition of $\bar{F}_i^{q,k}(\cdot)$, it holds that $\bar{F}_i^{q,0}(\mathbf{x}) = \bar{F}_j^{q,0}(\mathbf{x}),\,\forall i\neq j$ and $ \mathbf{x} \in \mathcal{M}$. Thus,~we denote $\bar{F}^{q,0}(\cdot) = \bar{F}_i^{q,0}(\cdot),\forall i\in[I]$. Moreover,~by~Algorithm~\ref{algo:full-info-feedback}, we can have $\bar{\mathbf{s}}^t = \bar{\mathbf{x}}^q(K+1),\,\forall (q-1)K+1 \le t\le qK$. As a result, it follows that,
     \begin{equation}\label{eq:regret-1-continuous-regret-decomp-2}
         \begin{aligned}
         \mathbb{E}[\mathcal{R}^c_T] &= \mathbb{E}\Big[K \sum_{q=1}^Q  \Big((1-1/e)\bar{F}^{q,0}(\mathbf{x}^\star) - \bar{F}^{q,0}\big(\bar{\mathbf{x}}^q(K+1)\big)\Big)\Big]\\
         & = \mathbb{E}\Big[K/I \sum_{q=1}^Q \sum_{i=1}^I  \Big((1-1/e)\bar{F}_i^{q,0}(\mathbf{x}^\star) - \bar{F}_i^{q,0}\big(\bar{\mathbf{x}}^q(K+1)\big)\Big)\Big]\\
          &\le  1/I\cdot \sum_{q=1}^Q \sum_{i=1}^I \sum_{k=1}^K \frac{1}{2\beta(k)}\mathbb{E} \Big[\|\nabla \bar{F}_i^{q,k-1}(\bar{\mathbf{x}}^q(k)) - \bar{\mathbf{d}}_i^q(k)\|^2 \Big]\\
        & \hspace{20pt}+ \sum_{q=1}^Q \sum_{k=1}^K(1-1/K)^{K-k} \cdot\mathbb{E}\big[\langle \widetilde{\mathbf{d}}^q(k), \mathbf{x}^\star - \bar{\mathbf{v}}^q(k) \rangle\big] \Big)+ \frac{QD^2}{2}\sum_{k=1}^K \beta(k) + \frac{L_2QD^2}{2},
         \end{aligned}
     \end{equation}
     where the inequality follows from Lemma~\ref{lemma:regret-1-key-step-1} and the fact that $\widetilde{\mathbf{d}}^q(k) = 1/I\cdot\sum_{i=1}^I\bar{\mathbf{d}}_i^q(k)$.

     Next, we deal with the term $\mathbb{E}[\langle \widetilde{\mathbf{d}}^q(k), \mathbf{x}^\star - \bar{\mathbf{v}}^q(k) \rangle]$ in~\eqref{eq:regret-1-continuous-regret-decomp-2}. By the definition of~$\bar{\mathbf{v}}^q(k)$, it holds that
     \begin{equation}
     \begin{aligned}
    &\mathbb{E} \Big[\langle \widetilde{\mathbf{d}}^q(k), \mathbf{x}^\star - \bar{\mathbf{v}}^q(k) \rangle \Big] \\
    &= \mathbb{E}\Big[ \sum_{i=1}^I\big\langle [\widetilde{\mathbf{d}}^q(k)]_i, [ \mathbf{x}^\star]_i- \mathbf{v}_{i}^q(k)\big\rangle \Big]  \\
    & =  \sum_{i=1}^I \mathbb{E} \Big[\big\langle [\widetilde{\mathbf{d}}^q(k)]_i \pm [\bar{\mathbf{d}}_i^q(k)]_i \pm [\mathbf{d}_i^q(k)]_i, [ \mathbf{x}^\star]_i - {\mathbf{v}}_i^q(k) \big\rangle\Big]\\
    & \le \sum_{i=1}^I \mathbb{E}\Big[\big\langle {[\mathbf{d}}_i^q(k)]_i, [ \mathbf{x}^\star]_i - {\mathbf{v}}_i^q(k) \big\rangle\big] + \sum_{i=1}^I D_i \cdot\mathbb{E}\big[\|[\bar{\mathbf{d}}_i^q(k)]_i - [{\mathbf{d}}_i^q(k)]_i\|\Big],
     \end{aligned}
     \end{equation}
  where the inequality is due to~\eqref{eq:independent-gradient} and $\|[ \mathbf{x}^\star]_i - {\mathbf{v}}_i^q(k)\| \le D_i$. Therefore, it follows that
  \begin{equation}\label{ineq:regret-1-continuous-regret-decomp-3}
      \begin{aligned}
              &\sum_{q=1}^Q \sum_{k=1}^K (1-1/K)^{K-k} \cdot\mathbb{E}\big[\langle \widetilde{\mathbf{d}}^q(k), \mathbf{x}^\star - \bar{\mathbf{v}}^q(k) \rangle\big]  \\
     & \le \sum_{q=1}^Q \sum_{k=1}^K \sum_{i=1}^I \mathbb{E} \Big[\big\langle {[\mathbf{d}}_i^q(k)]_i, [\mathbf{x}^\star]_i - {\mathbf{v}}_i^q(k) \big\rangle \Big] + \sum_{q=1}^Q \sum_{k=1}^K (1-1/K)^{K-k} \sum_{i=1}^I D_i\cdot\mathbb{E}\big[{\|[\bar{\mathbf{d}}_i^q(k)]_i - [{\mathbf{d}}_i^q(k)]_i\|}\big]  \\
     & \le  \sum_{q=1}^Q\sum_{i=1}^I \sum_{k=1}^K (1-1/K)^{K-k}  D_i\cdot\mathbb{E}\big[{\|\bar{\mathbf{d}}_i^q(k) - {\mathbf{d}}_i^q(k)\|}\big] +K\sum_{i=1}^I\mathcal{R}_Q^{\mathcal{E}_i},
      \end{aligned}
  \end{equation}
  where the last inequality is due to the definition of the regret of experts $\mathcal{E}^i$. Combining the inequalities~\eqref{eq:regret-1-continuous-regret-decomp-2} and~\eqref{ineq:regret-1-continuous-regret-decomp-3} completes the proof of Lemma~\ref{lemma:regret-1-intermediate-step}.
\end{proof}

 Subsequently, we turn to bound the two key terms in~\eqref{ineq:regret-1-intermediate-step}, i.e., the stochastic gradient error $\Delta_i^1(k)$ and the consensus error $\Delta_i^2(k)$. First, by the same steps in the proof of Lemma~6 in~\cite{zhang2019online}, we~can have the following result.
 \begin{lemma}\label{lemma:regret-1-stochastic-gradient-error-bound}
     Under conditions in Theorem~\ref{thm:regret-full-info}, the stochastic gradient error $\Delta_i^1(k) := \|\nabla F_i^{q,k-1}(\bar{\mathbf{x}}^q(k)) - \bar{\mathbf{d}}_i^q(k)\|^2$ has
     \begin{align}
    \mathbb{E}\big[\Delta_i^1(k)\big] \le \begin{cases}
      \frac{C_1}{(k+4)^{2/3}},\; &\;\; 1\le k\le K/2;\\
      \frac{C_1}{(K-k+1)^{2/3}},\; &\;\; K/2+1\le k\le K,\\
    \end{cases}
  \end{align}
  where $C_1 = 4L_1^2+32\cdot(2L_1+DL_2)^2$.
 \end{lemma}

 In addition, using the result in Lemma~\ref{lemma:consensus-full-info}, we can also bound the consensus error $\Delta_i^2(k)$ as the following Lemma~\ref{lemma:regret-1-consensus-error-bound}.
\begin{lemma}\label{lemma:regret-1-consensus-error-bound}
Under conditions in Theorem~\ref{thm:regret-full-info}, the consensus error $\Delta_i^2(k):=\|\bar{\mathbf{d}}_i^q(k) - {\mathbf{d}}_i^q(k)\|$ has
  \begin{align}
    \mathbb{E}\big[\Delta_i^2(k)\big] \le {L_2\kappa d(\mathcal{G})}/{K},\quad\forall k \in [K].
  \end{align}
\end{lemma}
\begin{proof}
  By Lemma~\ref{lemma:consensus-full-info} statement ii) and $L_2$-smoothness of each function $F^t(\cdot)$, it holds for $\forall t\in [T]$ that
  \begin{equation}\label{ineq:consensus-error-1}
    \begin{aligned}
    \|{\nabla} {F^{t}}(\bar{\mathbf{x}}^q(k)) - {\nabla} {F^{t}}(\mathbf{x}_i^q(k))\| &\le L_2 \|\bar{\mathbf{x}}^q(k) - \mathbf{x}_i^q(k)\|  \le {L_2\kappa d(\mathcal{G})}/{K}.
  \end{aligned}
  \end{equation}

  Next, we prove the statement in~Lemma~\ref{lemma:regret-1-consensus-error-bound} by mathematical induction. Recall that both $\bar{\mathbf{d}}^q(0)$ and ${\mathbf{d}}_i^q(0)$ are initialized~as zeros. By~\eqref{ineq:consensus-error-1}, it is easy to see that $\Delta_i^2(1) \le {L_2\kappa d(\mathcal{G})}/{K}$. Now, suppose that $\Delta_i^2(k-1) \le {L_2\kappa d(\mathcal{G})}/{K}$ holds, it follows by triangle inequality that
  \begin{equation}
  \begin{aligned}
    \Delta_i^2(k) &\le (1-\rho_k)\cdot \|\bar{\mathbf{d}}_i^q(k-1) - {\mathbf{d}}_i^q(k-1)\| + \rho_k \cdot\|{\nabla} F^{t_i^{q,k}}(\bar{\mathbf{x}}^q(k)) - {\nabla} F^{t_i^{q,k}}({\mathbf{x}}_i^q(k))\|  \le {L_2\kappa d(\mathcal{G})}/{K},
  \end{aligned}    
  \end{equation}
  which completes the proof.
\end{proof}

With the help of Lemmas~\ref{lemma:regret-1-stochastic-gradient-error-bound} and~\ref{lemma:regret-1-consensus-error-bound}, we are now able to~prove the statement in Theorem~\ref{thm:regret-full-info}. Let $\{\beta(k)\}_{k\in [K]}$ in Lemma~\ref{lemma:regret-1-key-step-1} be
  \begin{align}\label{eq:beta}
        \beta(k) = \begin{cases}
      (k+4)^{-1/3},\; &\text{when}\;\; 1\le k\le K/2;\\
      (K-k+1)^{-1/3},\; &\text{when}\;\;K/2+1\le k\le K,
    \end{cases}
  \end{align}
it holds that $\sum_{k=1}^{K/2} \beta(k) \hspace{-1pt}\le\hspace{-1pt} K^{2/3}$ and $\sum_{k=K/2+1}^{K} \beta(k) \hspace{-1pt}\le\hspace{-1pt} K^{2/3}$. Further, by Lemma~\ref{lemma:regret-1-stochastic-gradient-error-bound}, we have
\begin{equation}
\begin{aligned}
    \sum_{k=1}^{K} \frac{\mathbb{E}\big[\Delta_i^1(k)\big]}{\beta(k)} &= \sum_{k=1}^{K/2}  \frac{C_1}{(k+4)^{1/3}} + \sum_{k=K/2+1}^{K}  \frac{C_1}{(K-k+1)^{1/3}}  \le 2C_1 K^{2/3}.
\end{aligned}    
\end{equation}
Therefore, it follows from Lemmas~\ref{lemma:regret-1-intermediate-step} and~\ref{lemma:regret-1-consensus-error-bound} that
\begin{equation}
    \begin{aligned}
        \mathbb{E}[\mathcal{R}^c_T] &\le Q C_1 K^{2/3} + Q L_2\kappa d(\mathcal{G})\sum_{i=1}^I D_i + K\sum_{i=1}^I\mathcal{R}_Q^{\mathcal{E}_i} + QD^2K^{2/3} + L_2QD^2/2\\
        &\le \big(C_1 + C_0I+ D^2\big) \cdot T^{4/5} + \big(L_2D^2/2 + L_2\kappa d(\mathcal{G})\sum_{i=1}^I D_i\big)\cdot T^{2/5},
    \end{aligned}
\end{equation}
where the last inequality is due to $Q = T^{2/5}$, $K = T^{3/5}$, and also the regret bound of experts, i.e., $\mathcal{R}_Q^{\mathcal{E}_i} \le C_0 \sqrt{Q}$ (see~\eqref{ineq:expert-regret} in Sec.~\ref{subsec:FW-expert}). Finally, combining the fact that $\mathbb{E} [\mathcal{R}_T] \le \mathbb{E} [\mathcal{R}^c_T]$ in \eqref{eq:regret-1-continous-regret}, Theorem~\ref{thm:regret-full-info} has been proved.

\setcounter{subsection}{0}
\section*{Appendix II: Proofs Pertaining to Algorithm~\ref{algo:bandit-feedback}}
\addcontentsline{toc}{section}{Appendix II: Proofs Pertaining to Algorithm~\protect\ref{algo:bandit-feedback}}
\subsection{Proof of Lemma~\ref{lemma:consensus-bandit}}

 Due to the initialization of $\mathbf{x}_i^q(1)$ and the update of $[\mathbf{x}_i^q(k)]_i$ (see Lines 3 and 6 in Algorithm~\ref{algo:bandit-feedback}), it is straightforward to~have that $[\mathbf{x}_i^q(k+1)]_i = \sum_{l=1}^k \mathbf{v}_i^q(l)/K + (1-k/K) [\mathbf{x}_i^q(1)]_i$, which is convex combination of $[\mathbf{x}_i(1)]_i, \mathbf{v}_i^q(1), \cdots, \mathbf{v}_i^q(k)$. Provided that $[\mathbf{x}_i(1)]_i \in \widetilde{\mathcal{M}}_i$ and $\mathbf{v}_i^q(k) \in \widetilde{\mathcal{M}}_i ,\forall k\in[K]$, it follows by the convexity of set $\widetilde{\mathcal{M}}_i$ that $[\mathbf{x}_i^q(k+1)]_i \in \widetilde{\mathcal{M}}_i ,\forall k\in[K]$ and $i \in [I]$. The rest of this proof can be completed by following the same steps as the proof of Lemma~\ref{lemma:consensus-full-info} in Appendix I-A.

\subsection{Proof of Theorem~\ref{thm:regret-bandit}}
Similar to the previous proof, we first note that the~diameters of the polytopes~$\widetilde{\mathcal{M}}_i$ and $\widetilde{\mathcal{M}}$ are $\alpha D_i$~and $ \alpha D$, respectively. Given that $\alpha < 1$, $D_i$ and $D$ are always valid upper bounds of the diameters, and thus we directly use  $ D_i$~and $  D$ in the following proofs for simplicity. Furthermore, let us recall that $\gamma = \min_{i \in [I]}\big\{\min\{1, \kappa_i / \sqrt{N_i}\}\big\}$ and the definition of $\mathcal{M}$, it can be verified that $ \gamma \mathbb{B}^N_{\ge 0} \subseteq \mathcal{M}$. Therefore, according to Lemma 1 in \cite{zhang2019online}, with the parameters~$\alpha$ and $\delta$ specified~as~in Theorem~\ref{thm:regret-bandit}, the distance between the two sets $\mathcal{M}$ and $\widetilde{\mathcal{M}}$ has
\begin{align}\label{ineq:distance-between-two-sets}
    \text{dist}(\mathcal{M}, \widetilde{\mathcal{M}}) := \max_{\mathbf{x} \in \mathcal{M}, \mathbf{y} \in \widetilde{\mathcal{M}}} \;\|\mathbf{x} - \mathbf{y}\| \le C_\mathcal{M} \delta,
\end{align}
where $C_\mathcal{M} = (\sqrt{\kappa}/\gamma + 1)\sqrt{N} + \sqrt{\kappa}/\gamma$.

Next, we define a $\delta$-smoothed version of the function $F^t(\cdot)$, i.e., $\widehat{F}^t(\mathbf{x}):=\mathbb{E}_{\mathbf{v} \sim \text{Uni}(\mathbb{B}^N)} [F^t(\mathbf{x} + \delta \mathbf{v})]$, and based~on Lemma~1 in \cite{chen2020black}, we have the following general result regarding $\widehat{F}^t(\cdot)$.
\begin{lemma}\label{lemma:delta-smoothed-function}
    Suppose that $F^t(\cdot)$ is a $L_1$-Lipschitz and monotone continuous DR-submodular function, then so is $\widehat{F}^t(\cdot)$,~and for $\forall \mathbf{x} \in \widetilde{\mathcal{M}}$, it holds that $|\widehat{F}^t(\mathbf{x}) - {F}^t(\mathbf{x})| \le L_1\delta$.
\end{lemma}

With the help of the above Lemma~\ref{lemma:delta-smoothed-function}, we are now ready to analyze the regret in~Theorem~\ref{thm:regret-bandit}. Let $\mathcal{S}^\star$ and $\mathbf{x}^\star$ be defined as same as before and $\mathbf{x}_\delta^\star \in \arg\max_{\mathbf{x} \in \widetilde{\mathcal{M}}} \sum_{t=1}^T F^t(\mathbf{x})$, then by the same reasons behind the inequality~\eqref{eq:regret-1-continous-regret}, we have
\begin{equation}\label{ineq:eq:regret-2-continuous-regret-decomp}
    \begin{aligned}
        \mathbb{E}[\mathcal{R}_T] &\le  \mathbb{E}\Big[(1-1/e)  \sum_{t=1}^T  F^t (\mathbf{x}^\star) - \sum_{t=1}^T F^t\big(\bar{\mathbf{s}}^t\big)\Big] \\
        & =\mathbb{E}\Big[(1-1/e) \sum_{t=1}^T \big(F^t(\mathbf{x}^\star) \pm F^t(\mathbf{x}_\delta^\star) \pm \widehat{F}^t(\mathbf{x}_\delta^\star)\big) - \sum_{t=1}^T \big(F^t(\bar{\mathbf{s}}^t) \pm \widehat{F}^t(\bar{\mathbf{s}}^t)\big) \Big]\\
       &\le \mathbb{E}\Big[(1-1/e) \sum_{t=1}^T \big(F^t(\mathbf{x}^\star) - F^t(\mathbf{x}_\delta^\star)\big) + \sum_{t=1}^T\big((1-1/e) \widehat{F}^t(\mathbf{x}_\delta^\star) - \widehat{F}^t(\bar{\mathbf{s}}^t)\big) \Big] + (2-1/e)TL_1\delta\\
       & \le \mathbb{E}\Big[\sum_{t=1}^T\big((1-1/e) \widehat{F}^t(\mathbf{x}_\delta^\star) - \widehat{F}^t(\bar{\mathbf{s}}^t)\big) \Big] + (1-1/e)TL_1  C_\mathcal{M} \delta + (2-1/e)TL_1\delta,
    \end{aligned}
\end{equation}
  where the second inequality follows from~Lemma~\ref{lemma:delta-smoothed-function} and the last one follows from the facts that each $F^t(\cdot)$ is $L_1$-Lipschitz and $\|\mathbf{x}^\star - \mathbf{x}_\delta^\star\| \le \text{dist}(\mathcal{M}, \widetilde{\mathcal{M}}) \le C_\mathcal{M}\delta$.

  Subsequently, let $\widehat{\mathcal{R}}^c_T: = \sum_{t=1}^T  \big((1-1/e)\widehat{F}^{t}(\mathbf{x}_\delta^\star) - \widehat{F}^{t}(\bar{\mathbf{s}}^t)\big)$ and following the same analysis in the proof of Theorem~\ref{thm:regret-full-info} in Appendix I-B, we provide the result analogous to Lemma~\ref{lemma:regret-1-intermediate-step}. To proceed, we first define $\bar{F}_i^{q,k}(\mathbf{x}): = \sum_{l=k+1}^L \widehat{F}^{t_i^{q,l}}(\mathbf{x})/(L-k)$ for $\forall k = 0,1,\cdots K-1$. In addition, we overload the vector $\bar{\mathbf{d}}_i^q(k) \in \mathbb{R}^N$ which is given by the following~recursion with initialization $\bar{\mathbf{d}}_i^q(0) = \mathbf{0}$,
  \begin{equation}\label{eq:d-bar-2}
  \begin{aligned}
    \bar{\mathbf{d}}_i^q(k) &= (1-\rho_k)\bar{\mathbf{d}}_i^q(k-1) + \rho_k \cdot N/\delta \cdot F^{t^i_{q,k}}\big(\bar{\mathbf{x}}^q(k) + \delta\mathbf{u}_i^q(k)\big) \cdot\mathbf{u}_i^q(k).
  \end{aligned}
  \end{equation}
  Note that the randomness of $\bar{\mathbf{d}}_i^q(k)$  defined in~\eqref{eq:d-bar-2} comes from the random permutation $\{t_{q,1}^i, t_{q,2}^i, \cdots, t_{q,L}^i\}$ and the uniform sample $\mathbf{u}_i^q(k)$ from the sphere $\mathbb{S}^{N-1}$. However, both samples are independent across all agents and iterations, and therefore, the equality~\eqref{eq:independent-gradient} still holds in this case. Now, we are ready to show the intermediate result for the upper bound of $\mathbb{E}[\widehat{\mathcal{R}}_T^c]$, whose proof can be completed by the same steps as the one for Lemma~\ref{lemma:regret-1-intermediate-step} and thus is omitted for simplicity.
  \begin{lemma}\label{lemma:regret-2-intermediate-step}
  Let $\widehat{\mathcal{R}}_T^c$ be defined as above, then it holds that
    \begin{equation}\label{ineq:regret-2-intermediate-step}
        \begin{aligned}
                \mathbb{E}[\widehat{\mathcal{R}}^c_T] 
                &\le L/K\cdot\sum_{q=1}^Q \sum_{i=1}^I\sum_{k=1}^K \frac{1}{2\beta(k)I} \mathbb{E}\Big[{\underbrace{\|\nabla \bar{F}_i^{q,k-1}(\bar{\mathbf{x}}^q(k)) - \bar{\mathbf{d}}_i^q(k)\|^2}_{:= \Delta_i^1(k) }}\Big] \\ 
    &\quad+ L/K\cdot\sum_{q=1}^Q \sum_{i=1}^I\sum_{k=1}^K (1-1/K)^{K-k}  D_i\cdot\mathbb{E}\Big[\underbrace{\|\bar{\mathbf{d}}_i^q(k) - {\mathbf{d}}_i^q(k)\|}_{:=\Delta_i^2(k) }\Big]\\
    &\quad+ L\sum_{i=1}^I\mathcal{R}_Q^{\mathcal{E}_i}+ \frac{LQD^2}{2K}\sum_{k=1}^K \beta(k) + \frac{LQD^2L_2}{2K},
        \end{aligned}
    \end{equation}
      where the regret of experts has $\mathcal{R}_Q^{\mathcal{E}_i} \le C_0 \sqrt{Q}$.
  \end{lemma}

   Similarly, we now turn to bound the two terms in~\eqref{ineq:regret-2-intermediate-step}, i.e., the stochastic gradient error $\Delta_i^1(k)$ and consensus error $\Delta_i^2(k)$.
 \begin{lemma}\label{lemma:regret-2-stochastic-gradient-error-bound}
     Under conditions in Theorem~\ref{thm:regret-bandit}, the stochastic gradient error $\Delta_i^1(k) := \|\nabla F_i^{q,k-1}(\bar{\mathbf{x}}^q(k)) - \bar{\mathbf{d}}_i^q(k)\|^2$ has
     \begin{align}
    \mathbb{E}\big[\Delta_i^1(k)\big] \le \frac{C_2\delta^{-2} + C_3}{(k+3)^{2/3}},
  \end{align}
  where $C_2 = 2\sqrt[3]{16} L_1^2N^2 $, $C_3 = \sqrt[3]{16}\big(2L_1^2 + (2L_1 + 3DL_2)^2\big)$.
 \end{lemma}
 \begin{proof}
     This proof follows the same steps as the one for Lemma~11 in~\cite{zhang2019online} and thus is omitted for simplicity.
 \end{proof}
 \begin{lemma}\label{lemma:regret-2-consensus-error-bound}
Under conditions in Theorem~\ref{thm:regret-bandit}, the consensus error $\Delta_i^2(k):=\|\bar{\mathbf{d}}_i^q(k) - {\mathbf{d}}_i^q(k)\|$ has
  \begin{align}
    \mathbb{E}\big[\Delta_i^2(k)\big] \le \frac{L_1\kappa d(\mathcal{G})N}{K\delta},\quad\forall k \in [K].
  \end{align}
\end{lemma}
\begin{proof}
      By Lemma~\ref{lemma:consensus-bandit} statement ii) and the $L_1$-Lipschitz continuity of each function $F^t(\cdot)$, it holds for $\forall t\in [T]$ that
      \begin{equation}\label{ineq:consensus-error-2}
          \begin{aligned}
            &\big|{F^{t}}\big(\bar{\mathbf{x}}^q(k) + \delta \mathbf{u}_i(k)\big) - {F^{t}}\big(\mathbf{x}_i^q(k)+ \delta \mathbf{u}_i(k)\big)\big| \le L_1 \|\bar{\mathbf{x}}^q(k) - \mathbf{x}_i^q(k)\| \le {L_1\kappa d(\mathcal{G})}/{K}.
            \end{aligned}
      \end{equation}

    We also prove Lemma~\ref{lemma:regret-2-consensus-error-bound} by mathematical induction. Due to the fact that $\mathbf{u}_i^k$ is sampled from the sphere, i.e., $\|\mathbf{u}_i^k\| = 1$, and inequality~\eqref{ineq:consensus-error-2}, it is easy to see that $\Delta_i^2(1) \le {L_1\kappa d(\mathcal{G})N}/{K\delta}$. Now, suppose that $\Delta_i^2(k-1) \le {L_1\kappa d(\mathcal{G})N}/{K\delta}$ holds, it follows from triangle inequality that
    \begin{equation}
      \begin{aligned}
    \Delta_i^2(k) &\le (1-\rho_k)\cdot \Delta_i^2(k-1)+ \rho_k  N/\delta\cdot \big|{F^{t_{q,k}^i}}\big(\bar{\mathbf{x}}^q(k) + \delta \mathbf{u}_i(k)\big) - {F^{t_{q,k}^i}}\big(\mathbf{x}_i^q(k)+ \delta \mathbf{u}_i(k)\big)\big|  \\
    &\le \frac{L_1\kappa d(\mathcal{G})N}{K\delta},
  \end{aligned}
    \end{equation}
  which completes the proof.
\end{proof}

At last, we prove the statement in Theorem~\ref{thm:regret-bandit}. Here,  we~let $\beta(k) = (k+3)^{-1/3}/\delta $, then it has $\sum_{k=1}^K \beta(k) \le 1.5 K^{2/3} /\delta$. In addition, by Lemma~\ref{lemma:regret-2-stochastic-gradient-error-bound}, it holds that
  \begin{align}
    \sum_{k=1}^K \frac{\mathbb{E}[\Delta_i^1(k)]}{\beta(k)} &\le (C_2\delta^{-2} + C_3)\delta \cdot \sum_{k=1}^K \frac{1}{(k+3)^{1/3}} \le \frac{3(C_2\delta^{-1} + C_3\delta) K^{2/3}}{2}.
  \end{align}
  Therefore, by~\eqref{ineq:eq:regret-2-continuous-regret-decomp} and Lemmas~\ref{lemma:regret-2-intermediate-step}--\ref{lemma:regret-2-consensus-error-bound}, we can have that
  \begin{equation}
      \begin{aligned}
          \mathbb{E}[\mathcal{R}_T] &\le \mathbb{E}[\widehat{\mathcal{R}}^c_T]+ (1-1/e)TL_1  C_\mathcal{M} \delta + (2-1/e)TL_1\delta\\
          &\le (1-1/e)TL_1  C_\mathcal{M} \delta + (2-1/e)TL_1\delta + {3}/{4}\cdot C_2\delta^{-1}L Q K^{-1/3} + {3}/{4}\cdot C_3\delta L Q K^{-1/3} \\
          &\hspace{10pt}+ \frac{LQ L_1 \kappa d(\mathcal{G}) N \sum_{i=1}^I D_i}{\delta K}+L\sum_{i=1}^I\mathcal{R}_Q^{\mathcal{E}_i}+ \frac{3}{4\delta}\cdot{LQD^2}K^{-1/3}+ \frac{LQD^2L_2}{2K}\\
          & \le \Big((1-1/e)L_1C_\mathcal{M}C_\delta + (2-1/e)L_1 C_\delta + 3/4\cdot C_2/C_\delta + C_0 I + 3/4\cdot D^2/C_\delta\Big)\cdot T^{8/9}\\
          & \hspace{10pt}+3/4\cdot C_3C_\delta\cdot  T^{2/3} + L_1 \kappa d(\mathcal{G})N \sum_{i=1}^I D_i/C_\delta \cdot T^{4/9}+L_2D^2 /2 \cdot T^{1/3},
      \end{aligned}
  \end{equation}
  where the last inequality follows from $L = T^{7/9}$, $K = T^{2/3}$, $\delta = C_\delta T^{-1/9}$ and  $\mathcal{R}_Q^{\mathcal{E}_i} \le C_0 \sqrt{Q}$. The proof of Theorem~\ref{thm:regret-bandit} is completed.

\setcounter{subsection}{0}
\section*{Appendix III: Proofs Pertaining to Algorithm~\ref{algo:vs}}
\addcontentsline{toc}{section}{Appendix III: Proofs Pertaining to Algorithm~\protect\ref{algo:vs}}

\subsection{Proof of Theorem~\ref{thm:pipage-rounding}}
According to Algorithm~\ref{algo:pipage-rounding}, the \texttt{B-SPR} procedure can be~represented by $\mathbf{y}^0 = \mathbf{x} \to \mathbf{y}^1 \to \cdots \to \mathbf{y}^\mathcal{T} = \mathbf{z}$. Then,~by the randomized swapping operation in Algorithm~\ref{algo:pipage-rounding}, it holds that $ \mathbf{1}_M^\top\mathbf{y}^{\tau -1} = \mathbf{1}_M^\top\mathbf{y}^\tau$, $ \tau \in [\mathcal{T}]$, which proves the statement i) in Theorem~\ref{thm:pipage-rounding}. Also, it is clear that, for each $ \mathbf{y}^{\tau -1} \to \mathbf{y}^\tau, \tau \in [\mathcal{T}]$, there is at least one more element becoming either $b_l$ or $b_u$. Since there are totally $M$ elements in the vector, the procedure must terminate in finite time $\mathcal{T} \le M-1$, and there is at most one element in the output vector $\mathbf{z} = \mathbf{y}^\mathcal{T}$ that is neither $b_l$~nor $b_u$. Thus, the statement~ii) in Theorem~\ref{thm:pipage-rounding} is proved.

Subsequently, to prove the statement iii), let us first notice the following directional convexity property of the multi-linear extension $F(\cdot)$ of any monotone submodular function~\cite{rezazadeh2023distributed}:~let $\mathbf{e}_n, \mathbf{e}_m \in \mathbb{R}^M$ be two basis vectors, then $F\big(\mathbf{x} + t(\mathbf{e}_n - \mathbf{e}_m)\big)$~is convex with respect to the variable $t$. As a result of the above convexity property and the randomized swapping, we can have
\begin{equation}
    \begin{aligned}
        &\mathbb{E}\big[F(\mathbf{y}^{\tau+1}) | \mathbf{y}^\tau\big] \\
        &= \frac{\Delta^\tau(n)}{\Delta^\tau(m)+\Delta^\tau(n)} F\Big(\mathbf{y}^\tau + \Delta^\tau(m)(\mathbf{e}_n - \mathbf{e}_m)\Big)  + \frac{\Delta^\tau(m)}{\Delta^\tau(m)+\Delta^\tau(n)} F\Big(\mathbf{y}^\tau - \Delta^\tau(n)(\mathbf{e}_n - \mathbf{e}_m)\Big) \\
        & \ge F(\mathbf{y}^\tau).
    \end{aligned}
\end{equation}
It follows by the law of iterated expectations that $\mathbb{E}\big[F(\mathbf{y}^{\tau})] \ge \mathbb{E}\big[F(\mathbf{y}^{\tau-1})], \forall \tau \in [\mathcal{T}]$, which proves the statement iii).

\subsection{Proof of Lemma~\ref{lemma:rounding-norm-bound}}
We begin the proof by introducing more notations. Denote $b_l = \delta$ and $b_u(k) =\delta + \alpha (k-1)/K $, $\forall 1\le k\le K$. Similar to the previous proof of Theorem~\ref{thm:pipage-rounding}, we let the \texttt{B-SPR} procedure executed by $\mathcal{R}_{\delta}^{\delta+\alpha k/K }\hspace{-2pt}(\cdot)$ be represented~by $\mathbf{y}^0 = \widetilde{\mathbf{z}}(k+1) \to \mathbf{y}^1 \to \cdots \to \mathbf{y}^\mathcal{T} = \mathbf{z}(k+1)$ with $\mathcal{T} \le M-1$ (by Theorem~\ref{thm:pipage-rounding}). Since $\mathbf{z}(k)$~is the output of $\mathcal{R}_{b_l}^{b_u(k)}(\cdot)$, it is clear that $b_l \mathbf{1}_M\preceq \mathbf{z}(k) \preceq b_u(k) \mathbf{1}_M$ and by Theorem~\ref{thm:pipage-rounding} statement~ii), there is at most one element in $\mathbf{z}(k)$ that is neither $b_l$ nor $b_u(k)$. Let us denote the index~of such an element in $\mathbf{z}(k)$ as~$c \in [M]$, i.e., $z_c(k) \in \big(b_l ,b_u(k)\big)$ (if not exist, take arbitrary $c \in [M]$).  Recall that $\widetilde{\mathbf{z}}(k+1) = \mathbf{z}(k) + \mathbf{v}(k)/K$, and during the \texttt{B-SPR} procedure, let $c_\tau$ track the position~of such an element in $\mathbf{y}^\tau$ (this will be clearer in the sequel).

To prove the statement in Lemma~\ref{lemma:rounding-norm-bound}, let us first consider~the case when $k \le M$. By the fact that $b_l \mathbf{1}_M\le \mathbf{z}(k) \le b_u(k) \mathbf{1}_M$ holds for $\forall k \in [K+1]$,  it follows that
  \begin{align}\label{eq:proof-diff-bound-1}
    \|\mathbf{z}(k+1) - \mathbf{z}(k)\|_1 \le M \big(b_u(k+1) - b_l\big) \le \alpha M^2  /K.
  \end{align}

Subsequently, we turn to the case when $k > M$. To facilitate the proof, we define the following three intervals, i.e.,~$\mathcal{A}^\tau: = \big[ \delta + \alpha (k-1-\tau) /K, \delta + \alpha k/K\big]$, $\mathcal{B}^\tau:= \big[\delta, \delta +  \alpha(\tau+1) /K\big]$, and $\mathcal{C}^\tau:=\big(\delta + \alpha(\tau+1)  /K,  \delta + \alpha (k-1-\tau) /K\big).$ Note that $\mathcal{A}^\tau\cup \mathcal{B}^\tau \cup \mathcal{C}^\tau =\big[b_l , b_u(k+1)\big]$ holds for $\forall 0 \le \tau \le \mathcal{T}$, and since $k > M \ge \tau +1$, all the three intervals are well-defined. Associated with the intervals, we define three additional sets:
\vspace{-10pt}
\begin{subequations}
    \begin{align}
        \mathcal{Y}_\mathcal{A}^\tau: &= \Big\{y^\tau(m) \vert y^\tau(m) \in \mathcal{A}^\tau, m \in [M]\Big\} \setminus \big\{y^\tau (c_\tau)\big\};\\
        \mathcal{Y}_\mathcal{B}^\tau: &= \Big\{y^\tau(m) \vert y^\tau(m) \in \mathcal{B}^\tau, m \in [M]\Big\} \setminus \big\{y^\tau (c_\tau)\big\};\\
        \mathcal{Y}_\mathcal{C}^\tau: &= \Big\{y^\tau(m) \vert y^\tau(m) \in \mathcal{C}^\tau, m \in [M]\Big\} \cup \big\{y^\tau (c_\tau)\big\}.
    \end{align}
\end{subequations}
Then, we show the following intermediate result.
\begin{lemma}\label{lemma:rounding-intermediate-step}
    Let $k > M$, $\mathcal{T} \le M-1$ and the sets $\mathcal{Y}_\mathcal{A}^\tau$, $\mathcal{Y}_\mathcal{B}^\tau$, and $\mathcal{Y}_\mathcal{C}^\tau$ be defined as above, the following statements holds for $\forall \tau \in [\mathcal{T}]$:
    \begin{enumerate}[i)]
        \item $\mathcal{Y}_\mathcal{C}^\tau = \{y^\tau (c_\tau)\}$;
        \item for any $A$ elements $y^\tau(m_1),y^\tau(m_2)\cdots, y^\tau(m_{A})\in \mathcal{Y}_\mathcal{A}^\tau$ with $A \le {A}^\tau_\text{max}: = |\mathcal{Y}_\mathcal{A}^\tau|$, there is 
        \begin{align}\label{ineq:lemma-rounding-A}
           \hspace{-20pt}\sum_{i=1}^{A} y^\tau(m_i) \ge A\delta + \frac{\alpha}{K}\big(Ak - \min\{\tau+A, {A}^\tau_\text{max}\} \big);
        \end{align}
        \item for any $B$ elements $y^\tau(m_1),y^\tau(m_2)\cdots, y^\tau(m_{B})\in \mathcal{Y}_\mathcal{B}^\tau$ with $B \le {B}^\tau_\text{max}: = |\mathcal{Y}_\mathcal{B}^\tau|$, there is 
        \begin{align}\label{ineq:lemma-rounding-B}
            \sum_{i=1}^{B} y^\tau(m_i) \le B\delta + \frac{\alpha}{K}\big( \min\{\tau+B, {B}^\tau_\text{max}\} \big).
        \end{align}
    \end{enumerate}
\end{lemma}
\begin{proof}
        We prove Lemma~\ref{lemma:rounding-intermediate-step} by mathematical induction. First, due to the  facts that $\mathbf{y}^0 =\widetilde{\mathbf{z}}(k+1)=\mathbf{z}(k) + \mathbf{v}(k)/K$~and $\mathbf{v}(k) \in [0,\alpha]^M$, it can be verified that all statements hold for $\tau = 0$. Now, suppose that the statements hold for any $\tau \in [\mathcal{T}]$, we prove the case  for $\tau+1$.

        Recall that, at each step of the stochastic rounding procedure $\mathcal{R}_{\delta}^{\delta+\alpha k/K}(\cdot)$, two elements $y^\tau(m), y^\tau(n) \in \big(b_l, b_u(k+1)\big)$ are uniformly randomly selected. Next, we discuss on all five possibilities of the selected elements $y^\tau(m)$ and $y^\tau(n)$.

        \textbf{Case i):} $y^\tau(m), y^\tau(n) \in \mathcal{Y}_\mathcal{A}^\tau$. By the randomized swapping rule in Algorithm~\ref{algo:pipage-rounding}, one can have that~\footnote{We denote $[y(m), y(n)] \in \{a, b\}$ to include both cases $[y(m), y(n)] = [a, b]$ and $[y(m), y(n)] = [b, a]$.}
        \begin{align}
            [y^{\tau+1}(m), y^{\tau+1}(n)] \in \big\{ b_u(k\hspace{-2pt}+\hspace{-2pt}1), y^{\tau}(m)\hspace{-1pt}+\hspace{-1pt}y^{\tau}(n) \hspace{-1pt}-\hspace{-1pt}b_u(k\hspace{-2pt}+\hspace{-2pt}1)\big\}.
        \end{align}
        Due to the facts~that the statement ii) is true for $\mathcal{Y}_\mathcal{A}^\tau$ and both $y^\tau(m)$ and $ y^\tau(n)$ are from $ \mathcal{A}^\tau$, it holds that 
        \begin{align}\label{ineq:case-i}
            2\delta + \alpha (2k-\tau-2) /K \le y^\tau(m) + y^\tau(n) \le 2\delta + 2\alpha k/K,
        \end{align}
        where the inequality on left-hand-side follows from~\eqref{ineq:lemma-rounding-A} and the one on right-hand-side is due to the definition of $ \mathcal{A}^\tau$. As a result, we can have that $\delta + \alpha (k-\tau -2) /K  \le y^{\tau}(m)+y^{\tau}(n) -b_u(k+1) \le\delta + \alpha k/K $, i.e., $y^{\tau}(m)+y^{\tau}(n) -b_u(k+1) \in \mathcal{A}^{\tau+1}$. Together with the fact that $b_u(k+1) \in \mathcal{A}^{\tau+1}$, it follows that $y^{\tau+1}(m), y^{\tau+1}(n) \in \mathcal{A}^{\tau+1}$. In other words, if both $y^\tau(m)$ and $ y^\tau(n)$ are selected from~$ \mathcal{Y}_\mathcal{A}^\tau$ and in this case $c_{\tau+1} = c_\tau$, it has been proved~that they remain in the set $\mathcal{Y}_\mathcal{A}^{\tau+1}$ and the other two sets $\mathcal{Y}_\mathcal{B}^{\tau+1}$ and $\mathcal{Y}_\mathcal{C}^{\tau+1}$ remain unchanged. Therefore, the statements i) and iii) naturally holds for $\tau+1$. Next, we continue to prove the statement ii).~By the above analysis, we can have that $A^{\tau+1}_\text{max} = A^{\tau}_\text{max}$. Thus, the worst-case scenario~for proving~inequality \eqref{ineq:lemma-rounding-A} is~that only the smaller one of $y^{\tau+1}(m)$ and $ y^{\tau+1}(n)$ is selected when counting the summation of $A$ elements from $\mathcal{Y}_\mathcal{A}^{\tau+1}$. Without loss of generality, let us assume that $y^{\tau+1}(m)$ is the smaller one, i.e., $y^{\tau+1}(m) = y^{\tau}(m)+y^{\tau}(n) -b_u(k+1)$. In this scenario, it holds that
        \begin{equation}
            \begin{aligned}
                &\sum_{i=1}^{A} y^{\tau+1}(m_i) =  \hspace{-10pt}\sum_{i=1, m_i \neq m}^A \hspace{-10pt}y^{\tau}(m_i) + y^{\tau}(m)+y^{\tau}(n) -b_u(k+1) \\
                &\ge (A+1)\delta + \frac{\alpha}{K}\Big((A+1)k - \min\{\tau+A+1, A^{\tau}_\text{max}\} \Big)- b_u(k+1)\\
                & = A\delta + \frac{\alpha}{K}\Big(Ak-\min\{\tau +1+A, A^{\tau+1}_\text{max}\}\Big),
            \end{aligned}
        \end{equation}
        where the inequality is due to~\eqref{ineq:lemma-rounding-A} for the case $\tau$, and the last equality follows from the definition of $b_u(k+1)$  and the fact that $A^{\tau+1}_\text{max} = A^{\tau}_\text{max}$. Hence, the statement ii) is proved.

        \textbf{Case ii):} $y^\tau(m), y^\tau(n) \in \mathcal{Y}_\mathcal{B}^\tau$. Due to the symmetry of the definitions of $\mathcal{A}^\tau$ and $\mathcal{B}^\tau$ as well as the statements ii) and iii), this proof~can~be~completed by following the same steps as in the proof for case i) and thus is omitted for simplicity.

        \textbf{Case iii):} $y^\tau(m) \hspace{-1pt}\in\hspace{-1pt} \mathcal{Y}_\mathcal{A}^\tau, y^\tau(n) \hspace{-1pt}\in\hspace{-1pt} \mathcal{Y}_\mathcal{B}^\tau$ (or conversely, $y^\tau(n) \hspace{-1pt}\in\hspace{-1pt} \mathcal{Y}_\mathcal{A}^\tau, y^\tau(m) \hspace{-1pt}\in\hspace{-1pt} \mathcal{Y}_\mathcal{B}^\tau$). By the randomized swapping rule, one can have~that $[y^{\tau+1}(m), y^{\tau+1}(n)] \in  \big\{ b_l, y^{\tau}(m)+y^{\tau}(n) -b_l\big\} \cup \big\{ b_u(k+1), y^{\tau}(m)+y^{\tau}(n) -b_u(k+1)\big\}$. Due to the definitions of $\mathcal{A}^\tau$ and $\mathcal{B}^\tau$, it holds that $y^{\tau}(m)+y^{\tau}(n) -b_l \ge y^{\tau}(m)$ and thus $y^{\tau}(m)+y^{\tau}(n) -b_l \in \mathcal{A}^\tau \subset \mathcal{A}^{\tau+1}$. Similarly, it also~holds that $y^{\tau}(m)+y^{\tau}(n) -b_u(k+1) \in \mathcal{B}^{\tau+1}$. As a consequence, we can conclude either $y^{\tau+1}(m) \in \mathcal{A}^{\tau+1}, y^{\tau+1}(n) \in \mathcal{B}^{\tau+1}$ or $y^{\tau+1}(n) \in \mathcal{A}^{\tau+1}, y^{\tau+1}(m) \in \mathcal{B}^{\tau+1}$. In both cases, the set $\mathcal{Y}_\mathcal{C}^{\tau+1}$ remains unchanged (with $c_{\tau+1} = c_\tau$), which proves the statement i). Next, also due to the symmetry of the statements ii) and iii), we here only prove the statement ii) and the proof of the statement iii) simply follows by the same steps. By the above analysis, the only change from $\mathcal{Y}_\mathcal{A}^\tau$ to $\mathcal{Y}_\mathcal{A}^{\tau+1}$ is that $y^\tau(m)$ becomes either $ y^{\tau}(m)+y^{\tau}(n) -b_l$ or $ b_u(k+1)$. Since both quantities are no less than $y^\tau(m)$, the inequality~\eqref{ineq:lemma-rounding-A} follows for $\tau+1$ and the proof of statement ii) is finished.

        \textbf{Case iv):} $y^\tau(m) \hspace{-1pt}\in\hspace{-1pt} \mathcal{Y}_\mathcal{A}^\tau, y^\tau(n) \hspace{-1pt}\in\hspace{-1pt} \mathcal{Y}_\mathcal{C}^\tau$ (or conversely, $y^\tau(n) \hspace{-1pt}\in\hspace{-1pt} \mathcal{Y}_\mathcal{A}^\tau, y^\tau(m) \hspace{-1pt}\in\hspace{-1pt} \mathcal{Y}_\mathcal{C}^\tau$). Note that we have $c_\tau =n$ in this case. By the randomized swapping rule,  it is ensured that at least one of $y^{\tau+1}(m)$ and $y^{\tau+1}(n)$ is either $b_l$ or $b_u(k+1)$. Specifically, if $y^{\tau+1}(n)$ becomes $b_l$ or $b_u(k+1)$, we let $c_{\tau+1} =m$; otherwise $c_{\tau+1} =n$. Regardless, one can confirm that the statement i) holds in both cases. Subsequently, we prove the statements ii) and iii). By the above analysis, the change from $\mathcal{Y}_\mathcal{A}^\tau$ to $\mathcal{Y}_\mathcal{A}^{\tau+1}$ is that $y^\tau(m)$ either becomes  $ b_u(k+1)$ or is removed from the set $\mathcal{Y}_\mathcal{A}^{\tau}$. For the former case, since $b_u(k+1) \ge y^\tau(m)$ and the set $\mathcal{Y}_\mathcal{B}^{\tau+1}$ remains unchanged, the inequalities~\eqref{ineq:lemma-rounding-A} and~\eqref{ineq:lemma-rounding-B} holds for $\tau+1$. For the latter case, $b_l$ is added to the set~$\mathcal{Y}_\mathcal{B}^{\tau+1}$. Since $b_l = \delta$, one can also verify the inequality~\eqref{ineq:lemma-rounding-B} holds for $\tau+1$. Therefore, the proof in case iv) is completed.

        \textbf{Case v):} $y^\tau(m) \hspace{-1pt}\in\hspace{-1pt} \mathcal{Y}_\mathcal{B}^\tau, y^\tau(n) \hspace{-1pt}\in\hspace{-1pt} \mathcal{Y}_\mathcal{C}^\tau$ (or conversely, $y^\tau(n) \hspace{-1pt}\in\hspace{-1pt} \mathcal{Y}_\mathcal{B}^\tau, y^\tau(m) \hspace{-1pt}\in\hspace{-1pt} \mathcal{Y}_\mathcal{C}^\tau$). Again, due to the same symmetry reason, this proof~can~be~completed by following the same steps as in the proof for case iv) and thus is omitted for simplicity.
\end{proof}

With the aid of Lemma~\ref{lemma:rounding-intermediate-step}, we are now ready to prove the result in~Lemma~\ref{lemma:rounding-norm-bound}. To proceed,~let us first show that,~according to the randomized swapping operation, the distance~between $\mathbf{y}^{\tau+1}$ and $\mathbf{y}^{\tau}$ at each step $0\le \tau \le \mathcal{T}-1$ has
\begin{equation}
\begin{aligned}
  \mathbb{E}\Big[\|\mathbf{y}^{\tau+1} -\mathbf{y}^{\tau}\|_1\Big] = \frac{4}{ 1/{\Delta}^\tau(m) + 1/\Delta^\tau(n) }.
\end{aligned}
\end{equation}
Therefore, one can conclude that $\mathbb{E}\big[\|\mathbf{y}^{\tau+1} -\mathbf{y}^{\tau}\|_1\big]$ achieves its upper bound when $\Delta^\tau(m)$ and $\Delta^\tau(n)$ take their maximum values. Note that Lemma~\ref{lemma:rounding-intermediate-step}~confirms that the random selection of $y^\tau(m)$ and $y^\tau(n)$ at each step $\tau$ occurs in one of the five cases~above. Subsequently, we upper bound $\mathbb{E}[\|\mathbf{y}^{\tau+1} -\mathbf{y}^{\tau} \|_1]$ by analyzing $\Delta^\tau(m)$ and $\Delta^\tau(n)$ in each of the five cases.

\textbf{Case i)}: Considering $y^\tau(m) \in \mathcal{A}^\tau$ and $y^\tau(n) \in \mathcal{A}^\tau$, one can confirm that ${\Delta^\tau(m)} \le b_u(k+1) - y^{\tau}(n)\le  \alpha(\tau+1) /K$ and ${\Delta^\tau(n)} \le b_u(k+1) - y^{\tau}(m)\le \alpha(\tau+1) /K$. As a consequence, it holds that $\mathbb{E}[\|\mathbf{y}^{\tau+1} -\mathbf{y}^{\tau} \|_1] \le 2\alpha(\tau+1)/K$.

\textbf{Case ii)}: Similarly, one can confirm that $ {\Delta^\tau(m)} \le y^{\tau}(m) - b_l \le\alpha(\tau+1) /K$ and ${\Delta^\tau(n)} \le y^{\tau}(n) - b_l\le \alpha(\tau+1) /K$. Therefore, $\mathbb{E}[\|\mathbf{y}^{\tau+1} -\mathbf{y}^{\tau} \|_1] \le2\alpha(\tau+1)/K$.

\textbf{Case iii)}: Since $y^\tau(m) \in \mathcal{A}^\tau$ and $y^\tau(n) \in \mathcal{B}^\tau$, we can have that $\Delta^\tau(m) \le y^\tau(m)-b_l \le \alpha k/K$~and $\Delta^\tau(n) \le y^\tau(n)-b_l \le  \alpha(\tau\hspace{-1pt}+\hspace{-1pt}1)/K$. Therefore, $\mathbb{E}[\|\mathbf{y}^{\tau+1} -\mathbf{y}^{\tau} \|_1]  \le 4\alpha(\tau+1)/K$.

\textbf{Case iv)}: Since $y^\tau(m) \in \mathcal{A}^\tau$, $y^\tau(n) \in \mathcal{C}^\tau$, we can have~that ${\Delta^\tau(m)} \le y^{\tau}(m) - b_l \le\alpha k/K$ and ${\Delta^\tau(n)} \le b_u(k\hspace{-1pt}+\hspace{-1pt}1) - y^{\tau}(m)\le \alpha(\tau\hspace{-1pt}+\hspace{-1pt}1) /K$. Thus, it holds that $\mathbb{E}[\|\mathbf{y}^{\tau+1} -\mathbf{y}^{\tau} \|_1] \le 4\alpha(\tau+1)/K$.

\textbf{Case v)}: Similarly, one can confirm that ${\Delta^\tau(m)} \le y^{\tau}(m) -b_l\le  \alpha(\tau\hspace{-1pt}+\hspace{-1pt}1)/K$ and ${\Delta^\tau(n)} \le b_u(k\hspace{-1pt}+\hspace{-1pt}1) - y^{\tau}(m)\le \alpha k/K$. Therefore, it holds that $\mathbb{E}[\|\mathbf{y}^{\tau+1} -\mathbf{y}^{\tau} \|_1] \le 4\alpha(\tau+1)/K$.

Now, combining the above {cases i) -- v)}, we can obtain that $\mathbb{E}[\|\mathbf{y}^{\tau+1} -\mathbf{y}^{\tau} \|_1] \le 4\alpha(\tau+1)/K$, which further implies that
\begin{equation}
  \begin{aligned}
     &\mathbb{E}\Big[\|\mathbf{z}(k+1) - \widetilde{\mathbf{z}}(k+1)\|_1\Big]  = \mathbb{E}\Big[\|\mathbf{y}^\mathcal{T} - \mathbf{y}^0\|_1\Big] \le  \mathbb{E}\Big[\sum_{\tau = 0}^{\mathcal{T}-1}\|\mathbf{y}^{\tau+1}- \mathbf{y}^\tau\|_1\Big] \le 2M(M-1)\alpha/ K,
  \end{aligned}
\end{equation}
where last inequality is due to the fact that $\mathcal{T} \le M-1$.~Further, by triangle inequality, it holds that
\begin{equation}
\begin{aligned}
  \mathbb{E}\Big[\|\mathbf{z}(k+1) - {\mathbf{z}}(k)\|_1\Big] &\le \mathbb{E}\Big[\|\mathbf{z}(k+1) - \widetilde{\mathbf{z}}(k+1)\|_1\Big]  +  \| \mathbf{v}(k)\|_1/K\\
  &\le  2M(M-1)\alpha/ K + \alpha M /K\le 2\alpha  M^2 /K,
\end{aligned}
\end{equation}
which, together with the result in~\eqref{eq:proof-diff-bound-1}, completes the proof of Lemma~\ref{lemma:rounding-norm-bound}.

\subsection{Proof of Lemma~\ref{lemma:consensus-vs}}
Similar to the proofs of Lemmas~\ref{lemma:consensus-full-info} and~\ref{lemma:consensus-bandit}, we first show that $[\mathbf{x}_i^q(k+1)]_i \in \widetilde{\mathcal{M}}_i ,\forall k\in[K]$ and $i \in [I]$. Recall the definition of $\widetilde{\mathcal{M}}_i = \alpha \widetilde{\mathcal{M}}_i + \delta\cdot \mathbf{1}_{N_i}$, to prove $[\mathbf{x}_i^q(k+1)]_i \in \widetilde{\mathcal{M}}_i$, it is equivalent to show $\big([\mathbf{x}_i^q(k+1)]_i - \delta \cdot\mathbf{1}_{N_i}\big)/\alpha \in \mathcal{M}_i$, i.e., to prove the following: i) $\delta \cdot \mathbf{1}_{N_i} \preceq [\mathbf{x}_i^q(k+1)]_i \preceq (\alpha + \delta)\cdot \mathbf{1}_{N_i}$; and ii) $0 \le \big(\mathbf{1}_{N_i}^\top [\mathbf{x}_i^q(k+1)]_i - \delta N_i\big)/\alpha \le \kappa_i$. The~inequalities in i) are easily verified by the reason that $[\mathbf{x}_i^q(k+1)]_i$ is the output of $\mathcal{R}_{\delta}^{\delta + \alpha k/K}(\cdot)$. In addition, by the statement~i) in~Theorem~\ref{thm:pipage-rounding}, it holds that
\begin{equation}\label{eq:summation-recursion}
    \begin{aligned}
        \mathbf{1}_{N_i}^\top[\mathbf{x}_i^q(k+1)]_i &= \mathbf{1}_{N_i}^\top[\mathbf{x}_i^q(k )]_i + (\mathbf{1}_{N_i}^\top \mathbf{v}_i^q(k) - \delta N_i)/K= \sum_{l=1}^k \mathbf{1}_{N_i}^\top \mathbf{v}_i^q(l)/K + \delta N_i (1-k/K).
    \end{aligned}
\end{equation}
Since $\mathbf{v}_i^q(l) \in \widetilde{\mathcal{M}}_i$, we have that $\delta N_i \le \mathbf{1}_{N_i}^\top \mathbf{v}_i^q(l) \le \alpha \kappa_i + \delta N_i$ and thus the inequalities in ii) are proved.

Next, due to the updating rule~of $[\mathbf{x}_i^q(k)]_{i'}$ for $\forall i' \neq i$ (see Line~7~in Algorithm~\ref{algo:pipage-rounding}), it holds that $[\mathbf{x}_i^q(k)]_{i'} = [\mathbf{x}_{i'}^q(k- t)]_{i'}$ for some $1 \le t \le d(\mathcal{G}) $. As a consequence, one can also have that $[\mathbf{x}_i^q(k)]_{i'} \in \mathcal{M}_{i'}$ for $\forall i' \neq i$. The proof of statement i) in Lemma~\ref{lemma:consensus-vs} is completed.

By Lemma~\ref{lemma:rounding-norm-bound}, it has been shown that for $\forall i\in [I],k \in [K]$,
\begin{align}
    \mathbb{E}\big[\big\| [{\mathbf{x}}_i^q(k+1)]_i - [{\mathbf{x}}_i^q(k)]_i\big \|_1 \big] \le 2\alpha N_i^2/K.
\end{align}
Recall that $d(\mathcal{G})$ represents the diameter of the graph, meaning that~every agent $j \neq i$ can be reached from the agent $i$ at most in $d(\mathcal{G})$-hop communications, thus it follows from the triangle inequality that for $\forall j \neq i, k\in[K]$
  \begin{align}
    \mathbb{E}\Big[\big\| [{\mathbf{x}}_j^q(k)]_i - [{\mathbf{x}}_i^q(k)]_i\big \|_1 \Big] \le 2\alpha N_i^2d(\mathcal{G})/K.
  \end{align}
  As a result, it holds for $\forall i\in [n]$ that
 \begin{equation}
  \begin{aligned}
  &\mathbb{E}\Big[\big\|  \bar{\mathbf{x}}^q(k) - {\mathbf{x}}_i^q(k)\big \| \Big] \le \mathbb{E}\Big[\big\|  \bar{\mathbf{x}}^q(k) - {\mathbf{x}}_i^q(k)\big \|_1 \Big] = \sum_{j\neq i} \mathbb{E}\Big[\big\| [{\mathbf{x}}_j^q(k)]_j - [{\mathbf{x}}_i^q(k)]_j\big \|_1 \Big]
  % & \le \sum_{j\neq i} 2\alpha N_i^2d(\mathcal{G})/K\\
  \le 2\alpha N^2d(\mathcal{G})/K.
  \end{aligned}
  \end{equation}
  Therefore, the proof of Lemma~\ref{lemma:consensus-vs} is completed.

\subsection{Proof of Theorem~\ref{thm:regret-vs}}
This proof can be completed by repeating the proof in Appendix~II-B. The only two differences are noted~as follows.

First, due to the incorporation of the \texttt{PB-SPR} procedure (see Line 6 in Algorithm~\ref{algo:vs}), the update of $[{\mathbf{x}}_i^q(k)]_i$ is different from the one in Algorithm~\ref{algo:bandit-feedback}. Specifically, we let $\widetilde{\mathbf{x}}_i^q(k+1) := [\mathbf{x}_i^q (k)]_i + (\mathbf{v}_i^q(k) - \delta\cdot\mathbf{1}_{N_i})/K$, which is corresponding to the updating step in Algorithm~\ref{algo:bandit-feedback}, while Algorithm~\ref{algo:vs} imposes an additional operation by $[{\mathbf{x}}_i^q(k+1)]_i = \mathcal{R}_\delta^{\delta+\alpha k/K} (\widetilde{\mathbf{x}}_i^q(k+1))$. However, according to statement iii) in Theorem~\ref{thm:pipage-rounding}, it has
\begin{align}
    \mathbb{E}\Big[F\big(\bar{\mathbf{x}}^q(k+1)\big)\Big] \ge F\big(\widetilde{\mathbf{x}}^q(k+1)\big),\; \forall k \in[K],
\end{align}
where $\widetilde{\mathbf{x}}^q(k+1) := [\widetilde{\mathbf{x}}_1^q(k+1)^\top,\cdots,\widetilde{\mathbf{x}}_I^q(k+1)^\top]^\top$. Therefore, all arguments in the proof of Theorem~\ref{thm:regret-bandit} naturally follows~by taking additional expectation with respect to the randomness~in the \texttt{PB-SPR} procedure.

Second, in the previous proof, when bounding the consensus error term $\Delta_i^2(k)$ (see Lemma~\ref{lemma:regret-2-consensus-error-bound}), the $L_1$-Lipschitz continuity of $F^t(\cdot)$ and the fact $\|\bar{\mathbf{x}}^q(k) - \mathbf{x}_i^q(k)\| \hspace{-1pt}\le\hspace{-1pt} \kappa d(\mathcal{G})/K$ are~used. However, since the \texttt{PB-SPR} procedure is now involved, the bound of difference between $\bar{\mathbf{x}}^q(k)$ and~$\mathbf{x}_i^q(k)$ becomes that $\mathbb{E}\big[\|\bar{\mathbf{x}}^q(k) - \mathbf{x}_i^q(k)\|\big] \hspace{-1pt}\le\hspace{-1pt} 2\alpha N^2 d(\mathcal{G})/K$ (see the statement ii) in Lemma~\ref{lemma:consensus-vs}). As a consequence, the constant $N^2 d(\mathcal{G})$ appears in the second last term of the regret bound~in Theorem~\ref{thm:regret-vs}.

\subsection{Proof of Theorem~\ref{thm:vs-complexity}}

By the equality~\eqref{eq:summation-recursion} and the fact that $\mathbf{1}_{N_i}^\top \mathbf{v}_i^q(l) \le \alpha \kappa_i + \delta N_i$, $\forall l\in[K]$, it follows that
\begin{align}
    \mathbf{1}_{N_i}^\top[\mathbf{x}_i^q(k)]_i \le \delta N_i + \alpha\kappa_i (k-1) /K.
\end{align}
Moreover, given that $[\mathbf{x}_i^q(k)]_i$ is the output of $\mathcal{R}_\delta^{\delta+\alpha (k-1)/K}(\cdot)$, there are at least $N_i - \kappa_i$~components of the vector $[\mathbf{x}_i^q(k)]_i$ equal to $\delta$. In addition, since~it~has been shown in Lemma~\ref{lemma:consensus-vs} that  $[\mathbf{x}_i^q(k)]_{i'} = [\mathbf{x}_{i'}^q(k- t)]_{i'}$ holds for some $1 \le t \le d(\mathcal{G}) $, the above statement is also true for $[\mathbf{x}_i^q(k)]_{i'},\forall i' \neq i$.

Let $\mathbf{p}_i^q(k) = \mathbf{x}_{i}^q(k) + \delta\mathbf{u}_{i}^q(k)$, based upon which the discrete decision $\mathcal{X}^q_i(k)$ is randomly generated. Then, it holds that
\begin{equation}\label{ineq:prob-bound}
    \begin{aligned}
        \mathbb{P} \big(\mathcal{X}_i^q(k) \notin \mathcal{I}\big) &= \mathbb{P} \Big(\bigcup_{i'=1}^I \big|[\mathcal{X}_i^q(k)]_{i'} \big| > \kappa_{i'} \Big)\\
        & \le \sum_{i'=1}^I\mathbb{P} \Big( \big|[\mathcal{X}_i^q(k)]_{i'} \big| > \kappa_{i'} \Big ) \le \sum_{i'=1}^I 1- (1-2\delta)^{N_{i'} - \kappa_{i'}} \\
        & \le 2 \delta(N - \kappa).
    \end{aligned}
\end{equation}
Note that in~\eqref{ineq:prob-bound}, we use $[\mathcal{X}_i^q(k)]_{i'}$ to denote the discrete~decision generated by $[\mathbf{p}_i^q(k)]_{i'}$, the first inequality is due to~the union bound, the second inequality follows from the fact that there are at least $N_{i'} - \kappa_{i'}$ components of each block $[\mathbf{p}_i^q(k)]_{i'}$ less than or equal to $2\delta$, and the last one is by the~Bernoulli inequality. Hence, the inequality~\eqref{ineq:vs-prob-bound} in Theorem~\ref{thm:vs-complexity} is proved by $\delta=C_\delta T^{-1/9}$ and~\eqref{ineq:vs-regret-bound} simply follows from the~facts that $\mathbb{E}\big[\mathbbm{1}\{\mathcal{X}^q_i(k) \notin \mathcal{I}\}\big] = \mathbb{P}\big(\mathcal{X}^q_i(k) \notin \mathcal{I}\big)$, $Q = T^{2/9}$, $K = T^{2/3}$.

\subsection{Proof of Proposition~\ref{prop:vs-lower-bound}}
Using the same notations as in Appendix III-E, we can have
\begin{equation}
\begin{aligned}
    \mathbb{P} \big(\mathcal{X}_i^q(k) \notin \mathcal{I}\big) &= \mathbb{P} \Big(\bigcup_{i'=1}^I \big|[\mathcal{X}_i^q(k)]_{i'} \big| > \kappa_{i'} \Big) \ge \mathbb{P} \big( \big|[\mathcal{X}_i^q(k)]_{i} \big| > \kappa_{i} \big). 
\end{aligned}
\end{equation}
Therefore, to prove the statement in Proposition~\ref{prop:vs-lower-bound}, it suffices~to show that
\begin{align}
    \sum_{q=1}^Q \sum_{k=1}^K \mathbb{P} \big( \big|[\mathcal{X}_i^q(k)]_{i} \big| > \kappa_{i} \big) \ge \Omega(T^{7/9}).
\end{align}

By the sum-preserving property of $\mathcal{P}(\cdot)$ and the condition that $\mathbf{1}_{N_i}^\top \mathbf{v}_i^q(k) = \delta N_i + \alpha\kappa_i$, it is straightforward to have~that $\mathbf{1}_{N_i}^\top [\mathbf{x}_i^q(k)]_i = \delta N_i + \alpha \kappa_i k/K$. Note that the discrete decision $[\mathcal{X}_i^q(k)]_{i}$ is generated by applying the random rounding on the variable $\mathbf{g}_i^q(k): =  [\mathbf{x}_i^q(k)]_i + \delta [\mathbf{u}_i]_i$, then one can have that $\mathbf{1}_{N_i}^\top \mathbf{g}_i^q(k) \le \delta N_i + \delta \sqrt{N_i} + \alpha \kappa_i k/K$. Next, we~show~an~intermediate result in the following lemma, which verifies that the probability $\mathbb{P} ( |[\mathcal{X}_i^q(k)]_{i} | > \kappa_{i} )$ is minimized when the vector $\mathbf{g}_i^q(k)$ is most majorized.

\begin{lemma}\label{lemma:prob-config}
    Suppose that $x_1, x_2, \cdots, x_M$ are~independent Bernoulli variables with $\mathbb{P}(x_i = 1) = p_i$. Subject to constraints $\sum_{i=1}^M p_i <  \kappa$ and $\delta \le  p_i \le 1$, the probability $\mathbb{P}(\sum_{i=1}^N x_i > \kappa)$ is minimized when $\mathbf{p}=[p_1, p_2,\cdots p_M]$ is most majorized, i.e., there is at most one element in $\mathbf{p}$ distinct from $\delta$ and $1$.
\end{lemma}

\begin{proof}
    This proof can be found in~\cite{Gleser1975on,hoeffding1956distribution}.
\end{proof}

Due~to the definitions of $\alpha$ and $\delta$ (see conditions in Theorem \ref{thm:regret-bandit}), one can have that $\delta N_i + \delta \sqrt{N_i} + \alpha \kappa_i k/K < \kappa_i$ $ ,\forall k \in [K]$. Moreover, let us denote $s_k = (\delta \sqrt{N_i} + \alpha \kappa_i k/K)/(1-\delta)$ and $\sigma_k = \lfloor s_k \rfloor$, i.e., one can assign at most $\sigma_k$ one's in the~vector $\mathbf{g}_i^q(k)$ to minimize the probability $\mathbb{P} ( |[\mathcal{X}_i^q(k)]_{i} | > \kappa_{i} )$. Now,~let $K_s= \min_{k \in [K]} \{k \,\vert\, \sigma_k = \kappa_i \hspace{-1pt}-\hspace{-1pt}1\}$, and as a result of Lemma~\ref{lemma:prob-config}, we can have that 
\begin{equation}
    \begin{aligned}
        &\sum_{k=1}^K \mathbb{P} \big( \big|[\mathcal{X}_i^q(k)]_{i} \big| > \kappa_{i} \big) \ge \sum_{k=K_s}^K \Big((1-\delta)(s_k -\sigma_k) + \delta \Big)\Big(1-(1-\delta)^{N_i -\kappa_i}\Big).
    \end{aligned}
\end{equation}
Notice that
\begin{equation}
    \begin{aligned}
        &\sum_{k=K_s}^K \Big((1-\delta)(s_k -\sigma_k) + \delta \Big) \ge \sum_{k=K_s}^K \Big(s_k -\sigma_k \Big) \ge \frac{\alpha\kappa_i /K}{1-\delta} \frac{(K- K_s +1)(K- K_s)}{2} \ge \Omega(K),
    \end{aligned}
\end{equation}
where the first inequality follows from $s_k - \sigma_k < 1$, the second inequality is due to $s_{k+1} - s_k = {\alpha\kappa_i}/(K(1-\delta))$, and the last inequality is due to the definition of $K_s$. In addition, provided that $1-(1-\delta)^{N_i -\kappa_i} \ge (N_i -\kappa_i) \delta - (N_i -\kappa_i)^2\delta^2/2$ and also $\delta = \Omega(T^{-1/9})$, it follows that
\begin{align}
    \sum_{q=1}^Q \sum_{k=1}^K \mathbb{P} \big( \big|[\mathcal{X}_i^q(k)]_{i} \big| > \kappa_{i} \big) \ge \Omega(QK\delta) = \Omega(T^{7/9}).
\end{align}
The proof of Proposition~\ref{prop:vs-lower-bound} is completed.

\end{document}